\documentclass{article} 
\usepackage{times}
\usepackage{fullpage}
\usepackage{hyperref}
\usepackage{url}
\usepackage{amsmath,amsthm}
\usepackage{nicefrac}
\usepackage[psamsfonts]{amssymb}
\usepackage{wrapfig}
\usepackage{graphicx}

\def\showauthornotes{0}
\def\showdraftbox{0}

\newcommand{\defeq}{\stackrel{\textup{def}}{=}}

\newtheorem{theorem}{Theorem}[section]
\newtheorem{lemma}[theorem]{Lemma}
\newtheorem{definition}[theorem]{Definition} 
\newtheorem{corollary}[theorem]{Corollary}

\newcommand{\nfrac}[2]{\nicefrac{#1}{#2}}
\def\abs#1{\left| #1 \right|}
\newcommand{\norm}[1]{\ensuremath{\left\lVert #1 \right\rVert}}

\newcommand{\pair}[1]{\left\langle{#1}\right\rangle} 

\newcommand\rea{\mathbb R}

\newcommand{\marginlabel}[1]%
{\mbox{}\marginpar{\it{\raggedleft\hspace{0pt}#1}}}
\newcommand\poly{\mbox{poly}}  

\DeclareMathOperator*{\argmax}{arg\,max}

\newcommand\calI{\mathcal{I}}

 \ifcsname ifcommentflag\endcsname\else
  \expandafter\let\csname ifcommentflag\expandafter\endcsname
                  \csname iffalse\endcsname
\fi

\ifnum\showauthornotes=1

\else

\fi

\ifnum\showauthornotes=1
\newcommand{\Authornote}[2]{{\sf\footnotesize\color{red}{[#1: #2]}}}
\newcommand{\Authoredit}[2]{{\sf\color{red}{[#1]}\color{blue}{#2}}}
\newcommand{\Authorcomment}[2]{{\sf \color{gray}{[#1: #2]}}}
\newcommand{\Authorfnote}[2]{\footnote{\color{red}{#1: #2}}}
\newcommand{\Authorfixme}[1]{\Authornote{#1}{\textbf{??}}}
\newcommand{\Authormarginmark}[1]{\marginpar{\textcolor{red}{\fbox{
#1:!}}}}
\else
\newcommand{\Authornote}[2]{}
\newcommand{\Authoredit}[2]{}
\newcommand{\Authorcomment}[2]{}
\newcommand{\Authorfnote}[2]{}
\newcommand{\Authorfixme}[1]{}
\newcommand{\Authormarginmark}[1]{}
\fi

\newlength{\pgmtab}  
\newcommand {\ELSE}{{\bf else\ }}
\newcommand {\IF}{{\bf if\ }}
\newcommand {\FOR}{{\bf for\ }}
\newcommand {\TO}{{\bf to\ }}

\newcommand {\WHILE}{{\bf while\ }}

\newcommand {\THEN}{\mbox{\bf then\ }}

\newcommand {\RETURN}{\mbox{\bf return\ }}

 {
	\begin{enumerate}}{\end{enumerate}}

\def\qedsketch{\ifmmode\Box\else{\unskip\nobreak\hfil
\penalty50\hskip1em\null\nobreak\hfil$\Box$
\parfillskip=0pt\finalhyphendemerits=0\endgraf}\fi}

\newenvironment{proofof}[1]{\begin{trivlist} \item {\bf Proof
#1:~~}}
  { \hfill $\Box$ \end{trivlist}}

\newlength{\tpush}
\newcommand{\handout}[5]{
   \noindent
   \begin{center}
   \framebox{ \vbox{ \hbox to \textwidth { {\bf \coursenum\ :\  \coursename} \hfill #5 }
       \vspace{3mm}
       \hbox to \textwidth { {\Large \hfill #2  \hfill} }
       \vspace{1mm}
       \hbox to \textwidth { {\it #3 \hfill #4} }
     }
   }
   \end{center}
   \vspace*{4mm}
   \newcommand{\lecturenum}{#1}
   \addcontentsline{toc}{chapter}{Lecture #1 -- #2}
}

\ifnum\showdraftbox=1

\else

\fi

\allowdisplaybreaks
\newlength{\algtopspace}
\newlength{\algpostcaptionspace}
\newcommand{\Anote}{\Authornote{A}}
\newcommand{\Rnote}{\Authornote{R}}
\newcommand{\Snote}{\Authornote{S}}

\newtheorem{claim}[theorem]{Claim}

\newdimen\pIR
\newcommand\StevesR{{\rm I\kern\pIR R}}
\def\Reals#1{\StevesR^{#1}}

\def\defeq{=}
\def\setof#1{\left\{#1  \right\}}
\def\sizeof#1{\left|#1  \right|}

\def\union{\cup}
\def\intersect{\cap}

\def\abs#1{\left|#1  \right|}

\def\norm#1{\left\| #1 \right\|}

\newcommand{\AND}{\quad \text{and} \quad}

\usepackage{alltt} 
\usepackage{caption,mdframed}
\usepackage{xspace}
\usepackage{verbatim}
\usepackage[inline]{enumitem}

\newcommand{\eps}{\epsilon}

\newenvironment{tight_enumerate}{
\begin{enumerate}
  \setlength{\itemsep}{0pt}
  \setlength{\topsep}{30pt}
  \setlength{\parskip}{0pt}
}{\end{enumerate}}

\newcommand{\isoipm}{\textsc{IsotonicIPM}\xspace}
\newcommand{\apxipm}{\textsc{ApproxIPM}\xspace}
\newcommand{\hsol}{\textsc{HessianSolve}\xspace}
\newcommand\opt{\ensuremath{\mathsf{OPT}}\xspace}
\newcommand{\feasiblestart}{\textsc{GoodStart}\xspace}
\newcommand{\piso}{\ensuremath{p\text{-}\textsc{ISO}}\xspace}
\newcommand\optiso{\ensuremath{\text{OPT}_{\piso}}\xspace}
\newcommand\optlin{\ensuremath{\text{OPT}_{\text{lin}}}\xspace}
\newcommand\optbnd{\ensuremath{\text{OPT}_{\text{bnd}}}\xspace}

\newcommand{\rankone}{\textsc{RankOneMore}\xspace}

\newcommand{\sddsol}{\textsc{SDDSolve}\xspace}

\newcommand{\blksol}{\textsc{BlockSolve}\xspace}

\newcommand{\sym}{\operatorname{sym}}

\newcommand{\valsup}{K}

\newcommand\bnddom{\ensuremath{\mathcal{D}_{\valsup}}}

\newcommand{\vecone}{\mathbf{1}}

\newcommand{\assign}{\leftarrow}

\newcommand\odmfn{F}
\newcommand\vbfn{f_{\valsup}}
\newcommand\bndfn{F_{\valsup}}

\newcommand{\vecw}{\mathit{w^p}}

\newcommand{\wmaxp}{\mathit{w_{\text{max}}^p}}

\newcommand\len{\ensuremath{\mathsf{len}}}
\newcommand\grad{\ensuremath{\mathsf{grad}}}
\newcommand\lexless{\preceq_{\mathsf{lex}}}
\newcommand{\complexmin}{{\sc CompLexMin}}
\newcommand{\compinfmin}{{\sc CompInfMin}}

\begin{document}

\title{Fast, Provable Algorithms for Isotonic Regression in all
  $\ell_{p}$-norms
 \thanks{This research was partially supported
by AFOSR Award FA9550-12-1-0175, NSF grant CCF-1111257, and a Simons
Investigator Award to Daniel Spielman.}
 \thanks{Code from this work is available at
   \href{https://github.com/sachdevasushant/Isotonic}{\texttt{{https://github.com/sachdevasushant/Isotonic}}}.}
}

\author{
Rasmus Kyng \\
Dept. of Computer Science\\
Yale University\\
\texttt{rasmus.kyng@yale.edu} 
\and
Anup Rao\thanks{Part of this work was done when this author was a
  graduate student at Yale University.} \\
School of Computer Science\\
Georgia Tech\\
\texttt{arao89@gatech.edu} 
\and
Sushant Sachdeva \\
Dept. of Computer Science\\
Yale University\\
\texttt{sachdeva@cs.yale.edu}
}


\maketitle

\begin{abstract}
  Given a directed acyclic graph $G,$ and a set of values $y$ on the
  vertices, the Isotonic Regression of $y$ is a vector $x$ that
  respects the partial order described by $G,$ and minimizes
  $\norm{x-y},$ for a specified norm.  This paper gives improved
  algorithms for computing the Isotonic Regression for all weighted
  $\ell_{p}$-norms with rigorous performance guarantees. Our
  algorithms are quite practical, and variants of them can be
  implemented to run fast in practice.
\end{abstract}

\section{Introduction}
A directed acyclic graph (DAG) $G(V,E)$ defines a partial order on $V$
where $u$ precedes $v$ if there is a directed path from $u$ to $v.$
We say that a vector $x \in \rea^{V}$ is isotonic (with respect to
$G$) if it is a weakly order-preserving mapping of $V$ into $\rea.$ Let
$\calI_{G}$ denote the set of all $x$ that are isotonic with respect
to $G.$ It is immediate that $\calI_{G}$ can be equivalently defined
as follows:
\begin{equation}
\label{eq:intro:I-G}
\calI_{G} \defeq \{ x \in \rea^{V}\ |\ x_{u} \le x_{v} \textrm{ for
  all } (u,v) \in E\}.
\end{equation}
Given a DAG $G,$ and a norm $\norm{\cdot}$ on $\rea^{V},$ the
\emph{Isotonic Regression} of observations $y \in \rea^{V},$ is given
by $x \in \calI_{G}$ that minimizes $\norm{x-y}.$

Such monotonic relationships are fairly common in data. They allow one
to impose only weak assumptions on the data, e.g. the typical height
of a young girl child is an increasing function of her age, and the
heights of her parents, rather than a more constrained parametric model.

Isotonic Regression is an important shape-constrained nonparametric
regression method that has been studied since the 1950's~\cite{Ayer55,
  Barlow72, Gebhardt70}. It has applications in diverse fields such as
Operations Research~\cite{OR1, OR2} and Signal
Processing~\cite{Sig1}. In Statistics, it has several applications
(e.g.~\cite{Stats1, Stats2}), and the statistical properties of
Isotonic Regression under the $\ell_{2}$-norm have been well studied,
particularly over linear orderings (see~\cite{ChatterjeeGS15} and
references therein).  More recently, Isotonic regression has found
several applications in Learning~\cite{KalaiS09, Learning2, Learning4,
  ZadroznyE02, NarasimhanA13}. It was used by Kalai and
Sastry~\cite{KalaiS09} to provably learn Generalized Linear Models and
Single Index Models; and by Zadrozny and Elkan~\cite{ZadroznyE02}, and
Narasimhan and Agarwal~\cite{NarasimhanA13} towards constructing
binary Class Probability Estimation models.


The most common norms of interest are weighted $\ell_{p}$-norms,
defined as
\begin{equation*}
\norm{z}_{w,p} = 
\begin{cases}
\left(\sum_{v \in V}
w_{v}^{p} \cdot |z_{v}|^{p} \right)^{\nfrac{1}{p}}, & p \in [1,\infty),\\
\max_{v \in V} w_{v} \cdot |z_{v}|, & p = \infty,
\end{cases}
\end{equation*}
where $w_v > 0$ is the weight of a vertex $v \in V.$ In this paper, we
focus on algorithms for Isotonic Regression under weighted
$\ell_{p}$-norms. Such algorithms have been applied to large data-sets
from Microarrays~\cite{AngelovHKW06}, and from the
web~\cite{Learning3, Web2}.

Given a DAG $G,$ and observations $y \in \rea^{V},$ our regression
problem can be expressed as the following convex program:
\begin{equation}
\label{eq:intro:convex}
  \min \norm{x-y}_{w,p} \textrm{\quad such that }
  x_{u} \le x_{v} \textrm{ for all } (u,v) \in E.
\end{equation}

\subsection{Our Results}
Let $|V| = n,$ and
$|E| = m.$ We'll assume that $G$ is connected, and hence $m \ge n-1.$

\textbf{$\ell_{p}$-norms, $p < \infty$.} We give a unified,
optimization-based framework for algorithms that provably solve the
Isotonic Regression problem for $p \in [1,\infty).$ The following is
an informal statement of our main theorem (Theorem~\ref{thm:isoipm})
in this regard (assuming $w_v$ are bounded by $\poly(n)$).
\begin{theorem}[Informal]
\label{thm:intro:lp}
There is an algorithm that, given a DAG $G,$ observations $y,$ and
$\delta > 0,$ runs in time $O(m^{1.5} \log^{2} n \log \nfrac{n}{\delta}),$
and computes an isotonic $x_{\textrm{\sc Alg}} \in \calI_{G}$ such
that
\[\norm{x_{\textrm{\sc Alg}}-y}_{w,p}^{p} \le \min_{x \in \calI_{G}}\norm{x-y}_{w,p}^{p}
+ \delta.\]
\end{theorem}
 The previous best time bounds were $O(nm\log \frac{n^2}{m})$ for
$p \in (1,\infty)$~\cite{HochbaumQ03} and $O(nm + n^{2} \log n)$ for
$p=1$~\cite{StoutIsotonicPartitioning13}.

\textbf{$\ell_{\infty}$-norms.}
For $\ell_{\infty}$-norms, unlike $\ell_{p}$-norms for $p \in (1,\infty),$
the Isotonic Regression problem need not have a unique solution. There
are several specific solutions that have been studied in the
literature (see~\cite{StoutInfGeneral} for a detailed discussion). In this paper, we
show that some of them ({\sc Max, Min,} and {\sc Avg} to be precise) can be computed in time linear in
the size of $G$.
\begin{theorem}
\label{thm:intro:inf}
There is an algorithm that, given a DAG $G(V,E),$ a set of observations
$y \in \rea^{V},$ and weights $w,$ runs in expected time $O(m),$ and
computes an isotonic $x_{\textrm{\sc Inf}} \in \calI_{G}$ such that
\[\norm{x_{\textrm{\sc Inf}}-y}_{w,\infty} = \min_{x \in \calI_{G}}\norm{x-y}_{w,\infty}.\]
\end{theorem}
Our algorithm achieves the best possible running time. This was not known even for linear or tree orders. The
previous best running time was $O(m \log n)$~\cite{StoutInfGeneral}.

\textbf{Strict Isotonic Regression.}
We also give improved algorithms for Strict Isotonic Regression. Given
observations $y,$ and weights $w,$ its Strict Isotonic Regression
$x_{\textrm{\sc Strict}}$ is defined to be the limit of $\hat{x}_{p}$ as $p$ goes to $\infty$,
where $\hat{x}_{p}$ is the Isotonic Regression for $y$ under the norm
$\norm{\cdot}_{w,p}.$ It is immediate that $x_{\sc \textrm{Strict}}$
is an $\ell_{\infty}$ Isotonic Regression for $y.$ In addition, it is
unique and satisfies several desirable properties (see~\cite{StoutStrict12}).
\begin{theorem}
\label{thm:intro:strict}
There is an algorithm that, given a DAG $G(V,E),$ a set of observation
$y \in \rea^{V},$ and weights $w,$ runs in expected time $O(mn),$ and
computes $x_{\textrm{\sc Strict}},$ the strict Isotonic Regression of
$y.$
\end{theorem}
The previous best running time was
$O(\min(mn, n^{\omega}) + n^2 \log
n)$~\cite{StoutStrict12}.

\subsection{Detailed Comparison to Previous Results}
\textbf{$\ell_{p}$-norms, $p < \infty.$} There has been a lot of work
for fast algorithms for special graph families, mostly for $p=1,2$
(see~\cite{StoutSurvey} for references). For some cases where $G$ is
very simple, e.g. a directed path (corresponding to linear orders), or
a rooted, directed tree (corresponding to tree orders), several works
give algorithms with running times of $O(n)$ or $O(n \log n)$
(see~\cite{StoutSurvey} for references).


Theorem~\ref{thm:intro:lp} not only improves on the previously best
known algorithms for general DAGs, but also on several algorithms for
special graph families (see Table~\ref{table:comparison}). One such
setting is where $V$ is a point set in $d$-dimensions, and
$(u,v) \in E$ whenever $u_{i} \le v_{i}$ for all $i \in [d].$ This
setting has applications to data analysis, as in the example given
earlier, and has been studied extensively (see~\cite{StoutMulti15} for
references). For this case, it was proved by Stout (see
Prop. 2,~\cite{StoutMulti15}) that these partial orders can be
embedded in a DAG with $O(n \log^{d-1} n)$ vertices and edges, and
that this DAG can be computed in time linear in its size. The bounds
then follow by combining this result with our theorem above.


We obtain improved running times for all $\ell_{p}$ norms for DAGs
with $m = o(n^{2} / \log^6 n),$ and for $d$-dim point sets for
$d \ge 3.$ 
For $d=2,$ Stout~\cite{StoutIsotonicPartitioning13} gives
an $O(n \log^{2} n)$ time algorithm.


\begin{table}[t]
\caption{Comparison to previous best results for $\ell_{p}$-norms, $p \neq \infty$}
\label{table:comparison}
\begin{center}
\begin{tabular}{r | c c | c}
& 
\multicolumn{2}{c|}{\bf Previous best} & 
\multicolumn{1}{c}{\bf This paper} \\
& 
$\ell_{1}$ & 
$\ell_{p}, 1 < p < \infty$ &
$\ell_{p}, p < \infty $ \\ 
  \hline 
  $d$-dim vertex set, $d \ge 3$& 
$n^2 \log^{d} n$~\cite{StoutIsotonicPartitioning13} & 
$n^2 \log^{d+1} n$~\cite{StoutIsotonicPartitioning13} &
$n^{1.5} \log^{1.5(d+1)} n$\\
  arbitrary DAG &
 $nm + n^2 \log n$~\cite{AngelovHKW06}& 
$nm\log \frac{n^2}{m}$~\cite{HochbaumQ03}& 
 $m^{1.5} \log^{3} n$ \\
\end{tabular}
\end{center} {\small For sake of brevity, we have ignored the
  $O(\cdot)$ notation implicit in the bounds, and $o(\log n)$
  terms. 
  The results are reported assuming an error parameter
  $\delta = n^{-\Omega(1)},$ and that $w_v$ are bounded by $\poly(n).$}
\end{table}


\textbf{$\ell_{\infty}$-norms.}  
For weighted $\ell_{\infty}$-norms on
arbitrary DAGs, the previous best result was
$O( m \log n + n \log^2 n)$ due to Kaufman and
Tamir~\cite{KaufmanT93}. 
A manuscript by Stout~\cite{StoutInfGeneral}
improves it to $O(m \log n).$ These
algorithms are based on parametric search, and are impractical. 
Our algorithm is simple, achieves the best possible running time, 
and only requires random
sampling and topological sort.


In a parallel independent work, Stout~\cite{StoutInfSpecial15} gives
$O(n)$-time algorithms for linear order, trees, and $d$-grids, and an
$O(n \log^{d-1} n)$ algorithm for point sets in
$d$-dimensions. Theorem~\ref{thm:intro:inf} implies the linear-time
algorithms immediately. The result for $d$-dimensional point sets
follows after embedding the point sets into DAGs of size
$O(n \log^{d-1} n),$ as for $\ell_{p}$-norms.

\textbf{Strict Isotonic Regression.}
Strict Isotonic regression was introduced and studied
in~\cite{StoutStrict12}. It also gave the only previous algorithm for
computing it, that runs in time $O(\min(mn, n^{\omega}) + n^2 \log
n).$ Theorem~\ref{thm:intro:strict} is an improvement when
$m = o(n \log n)$.

\subsection{Overview of the Techniques and Contribution}
\textbf{$\ell_{p}$-norms, $p < \infty$.}
It is immediate that Isotonic Regression, as formulated in
Equation~\eqref{eq:intro:convex}, is a convex programming problem. For
weighted $\ell_{p}$-norms with $p < \infty,$ applying generic
convex-programming algorithms such as Interior Point methods to this
formulation leads to algorithms that are quite slow.

We obtain faster algorithms for Isotonic Regression by replacing the
computationally intensive component of Interior Point methods, solving
systems of linear equations, with approximate solves. This approach
has been used to design fast algorithms for generalized flow
problems~\cite{DaitchS08, Madry13, LeeS14}. 

We present a complete proof of an Interior Point method for a large
class of convex programs that only requires approximate solves.
Daitch and Spielman~\cite{DaitchS08} had proved such a result for
linear programs.  We extend this to $\ell_{p}$-objectives, and provide
an improved analysis that only requires linear solvers with a constant
factor relative error bound, whereas the method from Daitch and
Spielman required polynomially small error bounds.

The linear systems in~\cite{Madry13, LeeS14} are Symmetric Diagonally
Dominant (SDD) matrices. The seminal work of Spielman and
Teng~\cite{SpielmanT04} gives near-linear time approximate solvers for
such systems, and later research has improved these solvers
further~\cite{KoutisMP11, CohenSTOC}.  Daitch and
Spielman~\cite{DaitchS08} extended these solvers to M-matrices
(generalizations of SDD). The systems we need to solve are neither
SDD, nor M-matrices.  We develop fast solvers for this new class of
matrices using fast SDD solvers. We stress that standard techniques
for approximate inverse computation, e.g. Conjugate Gradient, are not
sufficient for approximately solving our systems in near-linear
time. These methods have at least a square root dependence on the
condition number, which inevitably becomes huge in IPMs.

\textbf{$\ell_{\infty}$-norms and Strict Isotonic Regression.}
Algorithms for $\ell_{\infty}$-norms and Strict Isotonic Regression
are based on techniques presented in a recent paper of Kyng \emph{et
  al.}~\cite{KyngRSS15}.  We reduce $\ell_{\infty}$-norm Isotonic
Regression to the following problem, referred to as Lipschitz learning
on directed graphs in~\cite{KyngRSS15} (see Section~\ref{sec:inf} for
details) : We have a directed graph $H,$ with edge lengths given by
$\len.$ Given $x \in \rea^{V(H)},$ for every $(u,v) \in E(H),$ define
$\grad^{+}_{G}[x](u,v) \defeq \max \left \{
  \frac{x(u)-x(v)}{\len(u,v)}, 0 \right \}.$
Now, given $y$ that assigns real values to a subset of $V(H),$ the
goal is to determine $x \in \rea^{V(H)}$ that agrees with $y$ and
minimizes $\max_{(u,v) \in E(H)} \grad^{+}_{G}[x](u,v).$

The above problem is solved in $O(m + n\log n)$ time for general
directed graphs in~\cite{KyngRSS15}. We give a simple linear-time
reduction to the above problem with the additional property that $H$
is a DAG. For DAGs, their algorithm can be implemented to run in
$O(m+n)$ time.

It is proved in~\cite{StoutStrict12} that computing the Strict
Isotonic Regression is equivalent to computing the isotonic vector
that minimizes the error under the lexicographic ordering (see
Section~\ref{sec:inf}).  
Under the same reduction as in the
$\ell_{\infty}$-case, we show that this is equivalent to minimizing
$\grad^{+}$ under the lexicographic ordering. 
It is proved in~\cite{KyngRSS15} that the lex-minimizer can be
computed with basically $n$ calls to $\ell_{\infty}$-minimization,
immediately implying our result.



\subsection{Further Applications}
The IPM framework that we introduce to design our algorithm for
Isotonic Regression (IR), and the associated results, are very general, and
can be applied as-is to other problems. As a concrete application, the
algorithm of Kakade et al.~\cite{Learning4} for provably learning
Generalized Linear Models and Single Index Models learns 1-Lipschitz
monotone functions on linear orders in $O(n^2)$ time (procedure
LPAV). The structure of the associated convex program resembles
IR. Our IPM results and solvers immediately imply an $n^{1.5}$ time
algorithm (up to log factors).

Improved algorithms for IR (or for learning Lipschitz functions) on
$d$-dimensional point sets could be applied towards learning d-dim
multi-index models where the link-function is nondecreasing w.r.t. the
natural ordering on d-variables, extending \cite{KalaiS09,
  Learning4}. They could also be applied towards constructing Class
Probability Estimation (CPE) models from multiple classifiers, by
finding a mapping from multiple classifier scores to a probabilistic
estimate, extending~\cite{ZadroznyE02, NarasimhanA13}.

\smallskip \textbf{Organization.}  We report experimental results in
Section~\ref{sec:expts}. An outline of the algorithms and
analysis for $\ell_{p}$-norms, $p < \infty,$ are presented in
Section~\ref{sec:ipm}. In Section~\ref{sec:inf}, we define the
Lipschitz regression problem on DAGs, and give the reduction from
$\ell_{\infty}$-norm Isotonic Regression. We defer a detailed
description of the algorithms, and most proofs to the accompanying
supplementary material.

\vspace{-5pt}
\section{Experiments}
\label{sec:expts}
An important advantage of our algorithms is that they can be
implemented quite efficiently. Our algorithms are based on what is
known as a short-step method (see Chapter 11,~\cite{BoydV04}), that
leads to an $O(\sqrt{m})$ bound on the number of iterations. Each
iteration corresponds to one linear solve in the Hessian matrix. A
variant, known as the long-step method (see~\cite{BoydV04}) typically
require much fewer iterations, about $\log m,$ even though the only
provable bound known is $O(m).$
\begin{wrapfigure}{H}{0.5\textwidth}
\vspace{-22pt}
  \begin{center}
    \includegraphics[width=0.48\textwidth]{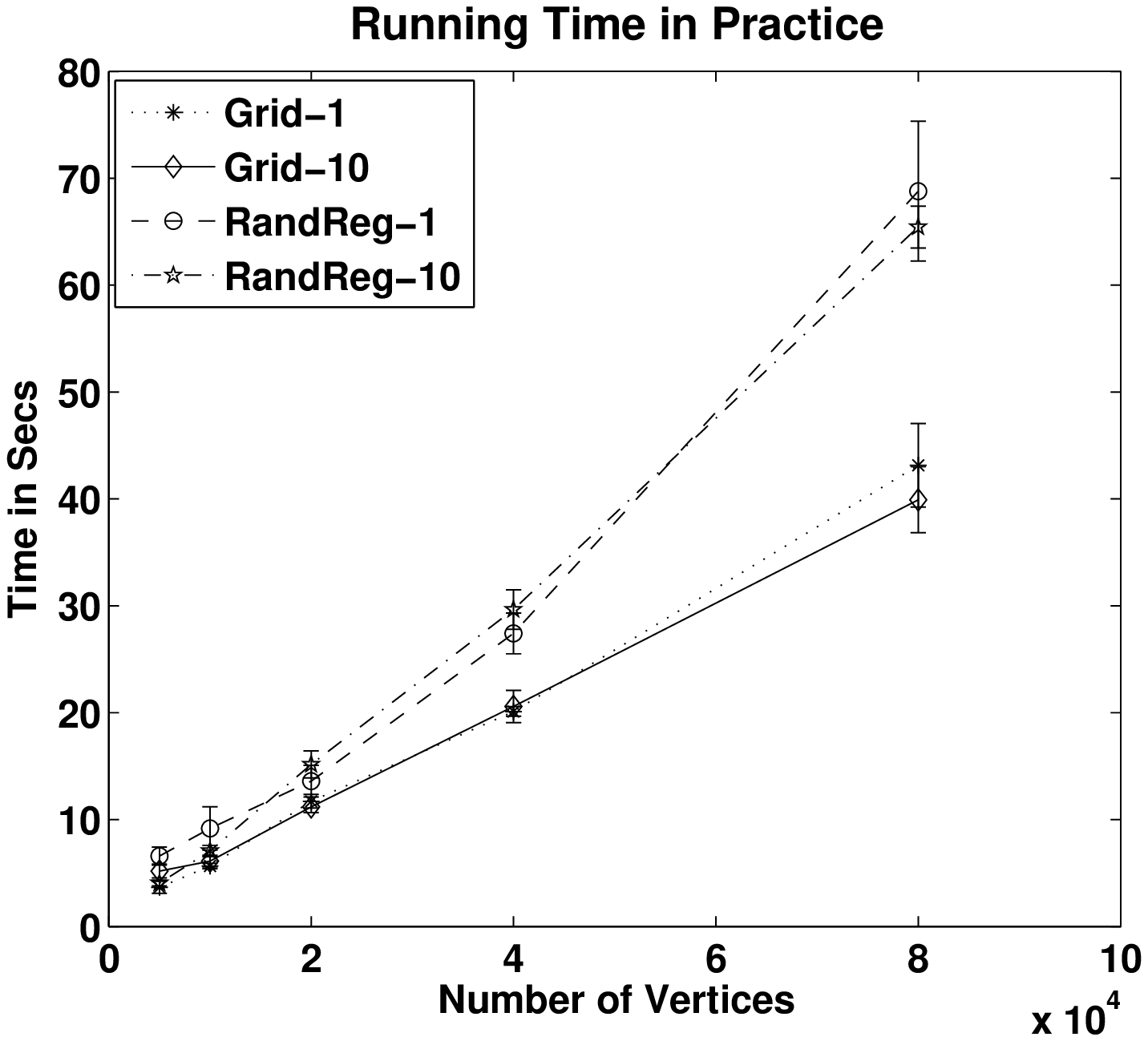}
\vspace{-8pt}
  \end{center}
    \label{fig:runtimes}
    \vspace{-14pt}
\end{wrapfigure}

For the important special case of $\ell_{2}$-Isotonic Regression, we
have implemented our algorithm in Matlab, with long step barrier
method, combined with our approximate solver for the linear systems
involved. A number of heuristics recommended in~\cite{BoydV04} that
greatly improve the running time in practice have also been
incorporated. Despite the changes, our implementation is theoretically
correct and also outputs an upper bound on the error by giving a
feasible point to the dual program.  Our implementation is
available at
\href{https://github.com/sachdevasushant/Isotonic}{\texttt{{https://github.com/
      sachdevasushant/Isotonic}}}.

In the figure, we plot average running times (with error bars denoting
standard deviation) for $\ell_{2}$-Isotonic Regression on DAGs, where
the underlying graphs are 2-d grid graphs and random regular graphs
(of constant degree). The edges for 2-d grid graphs are all oriented
towards one of the corners. For random regular graphs, the edges are
oriented according to a random permutation. The vector of initial
observations $y$ is chosen to be a random permutation of $1$ to $n$
obeying the partial order, perturbed by adding i.i.d. Gaussian noise
to each coordinate. For each graph size, and two different noise
levels (standard deviation for the noise on each coordinate being $1$
or $10$), the experiment is repeated multiple time.
The relative error in
the objective was ascertained to be less than 1\%.




\setcounter{table}{0}
\renewcommand{\tablename}{Algorithm}
\section{Algorithms for $\ell_{p}$-norms, $p < \infty$}
\label{sec:ipm}
Without loss of generality, we assume $y \in [0,1]^{n}.$
Given $p \in [1,\infty),$ let $\piso$ denote the following
$\ell_{p}$-norm Isotonic Regression problem, and $\optiso$ denote its optimum:
\begin{equation}
\label{eq:isoprog}
\begin{aligned}
& \underset{x\in \calI_{G} }{\min} & &\norm{x-y}_{w,p}^{p}.
\end{aligned}
\end{equation}
Let $\vecw$ denote the entry-wise $p^{\text{th}}$ power of $w$.  We
assume the minimum entry of $\vecw$ is $1,$ and the maximum entry is
$\wmaxp \leq \exp(n)$.  We
also assume the additive error parameter $\delta$ is lower bounded by
$\exp(-n)$, and that $p \leq \exp(n)$. We use the $\widetilde{O}$
notation to hide $\poly \log \log n$ factors.
\begin{theorem}
\label{thm:isoipm}
Given a DAG $G(V,E),$ a set of observations $y \in [0,1]^{V},$ weights
$w,$ and an error parameter $\delta >0,$ the algorithm \isoipm runs in
time
$\widetilde{O}\left(m^{1.5} \log^2 n \log\left( \nfrac{n p
      \wmaxp}{\delta} \right)\right),$
and with probability at least $1-\nfrac{1}{n},$ outputs a vector
$x_{\textrm{\sc Alg}} \in \calI_{G}$ with
  \[
  \norm{x_{\textrm{\sc Alg}}-y}_{w,p}^p \leq
\optiso + \delta.
  \]
\end{theorem}
%
The algorithm \isoipm is obtained by an appropriate instantiation of a
general Interior Point Method (IPM) framework which we call \apxipm. 

To state the general IPM result, we need to introduce two important
concepts.  These concepts are defined formally in Supplementary
Material Section~\ref{sec:ipmdefs}.  The first concept is
\emph{self-concordant barrier functions}; we denote the class of these
functions by $\mathcal{SCB}$.  A self-concordant barrier function $f$
is a special convex function defined on some convex domain set
$S$. The function approaches infinity at the boundary of $S$.  We
associate with each $f$ a \emph{complexity parameter} $\theta(f)$
which measures how well-behaved $f$ is.  The second important concept
is the \emph{symmetry} of a point $z$ w.r.t.  $S$: A
non-negative scalar quantity $\sym(z,S)$.  A large symmetry value
guarantees that a point is not too close to the boundary of the set.
For our algorithms to work, we need a starting point whose symmetry is
not too small.  We later show that such a starting point can be
constructed for the \piso problem.

\apxipm is a primal path following IPM: Given a vector $c$, a domain
$D$ and a barrier function $f \in \mathcal{SCB}$ for $D$, we seek to
compute $\min_{x \in D} \pair{c,x}.$ To find a minimizer, we consider
a function $ f_{c,\gamma}(x) = f(x) + \gamma \pair{c, x}$, and attempt
to minimize $f_{c,\gamma}$ for changing values of $\gamma$ by
alternately updating $x$ and $\gamma$.  As $x$ approaches the boundary
of $D$ the $f(x)$ term grows to infinity and with some care, we can
use this to ensure we never move to a point $x$ outside the feasible
domain $D$.  As we increase $\gamma$, the objective term $\pair{c, x}$
contributes more to $f_{c,\gamma}$.  Eventually, for large enough
$\gamma$, the objective value $\pair{c,x}$ of the current point $x$
will be close to the optimum of the program.

To stay near the optimum $x$ for each new value of $\gamma$,
we use a second-order method (Newton steps) to update $x$
when $\gamma$ is changed.
This means that we minimize a local quadratic approximation to our objective.
This requires solving a linear system $H z = g$,
where $g$ and $H$ are the gradient and Hessian of $f$ at $x$
respectively.
Solving this system to find $z$ is the most computationally intensive
aspect of the algorithm.
Crucially we ensure that crude approximate solutions to the linear system suffices,
allowing the algorithm to use fast approximate solvers for this step.
\apxipm is described in detail in Supplementary Material
Section~\ref{sec:apxipm},
 and in this section we prove the following theorem.
\begin{theorem}
Given a convex bounded domain $D \subseteq \Reals{n}$ and
  vector $c \in \Reals{n}$, consider the program
\begin{equation}
\begin{aligned}
& \underset{x \in D}{\min}
& & \pair{c,x}.
\end{aligned}
\label{eq:genericLP}
\end{equation}
Let \opt denote the optimum of the program.  Let
$f \in \mathcal{SCB}$ be a self-concordant barrier function for $D$.
Given a initial point $x_{0} \in D$, a value upper bound
$\valsup \geq \sup\{ \pair{c,x} : x \in D\} $, a symmetry lower bound
$ s \leq \sym(x_{0}, D)$, and an error parameter $0< \eps < 1$, the
algorithm \apxipm runs for
  \[
\textstyle T_{\text{apx}} = O\left( \sqrt{\theta(f)} \log\left( \nfrac{\theta(f)}{\epsilon \cdot s } \right) \right)
  \]
  iterations and returns a point $x_{\text{apx}}$,
  which satisfies
  $
\textstyle   \frac{\pair{c,x_{\text{apx}}} - \opt}{\valsup - \opt} \leq \epsilon.
  $

  The algorithm requires $O(T_{\text{apx}})$ multiplications of
  vectors by a matrix $M(x)$ satisfying
  $\nfrac{9}{10} \cdot H(x)^{-1} \preceq M(x) \preceq \nfrac{11}{10}
  \cdot H(x)^{-1},$
  where $H(x)$ is the Hessian of $f$ at various points $x \in D$
  specified by the algorithm.
\label{thm:pathalg}
\end{theorem}
We now reformulate the \piso program to
state a version which can be solved using
the \apxipm framework.
Consider points $(x,t) \in \Reals{n} \times \Reals{n}$,
and define a set 
  \[
  D_{G} = \setof{ (x,t) : \text{ for all } v \in V \, . \, \abs{x(v) - y(v)}^{p} - t(v) \leq 0 }.
  \]
  To ensure boundedness, as required by {\apxipm}, we add the constraint
  $\pair{\vecw,t} \leq \valsup$.
 \begin{definition}
   We define the domain
$
\bnddom = (\calI_{G} \times \Reals{n}) \intersect  D_{G} \intersect 
\setof{ (x,t) : \pair{\vecw,t} \leq \valsup }.
$
\end{definition} 
The domain $\bnddom$ is convex, and allows us to reformulate
program~\eqref{eq:isoprog} with
a linear objective:
  \begin{equation}
\label{eq:bndisoprog}
\underset{x,t}{\min}  \pair{\vecw,t} \quad \text{such that }  (x,t) \in \bnddom.
\end{equation}
 Our next lemma determines a choice of $K$ which suffices to ensure that programs~\eqref{eq:isoprog} and~\eqref{eq:bndisoprog} have the same optimum. The lemma is proven in Supplementary Material Section~\ref{sec:startpoint}.
 \begin{lemma}
   For all $\valsup \geq 3 n \wmaxp $, 
   $\bnddom$ is non-empty and
   bounded, and the optimum of program~\eqref{eq:bndisoprog} is $\optiso$.
  \label{lem:boundeddom}
\end{lemma} 
The following result shows that for program~\eqref{eq:bndisoprog} we
can compute a good starting point for the path following IPM efficiently.
The algorithm \feasiblestart computes a starting point
in linear time
by running a topological sort on the vertices of the DAG $G$
and assigning values to $x$ according to the vertex order of the sort.
Combined with an appropriate choice of $t$,
this suffices to give a starting point with good symmetry.
The algorithm \feasiblestart is specified in more detail in
Supplementary Material Section~\ref{sec:startpoint},
together with a proof of the following lemma.
\begin{lemma}
\label{lem:startsym}
The algorithm \feasiblestart runs in time $O(m)$ and returns an initial point $(x_{0},t_{0})$
that is feasible, and   for $\valsup =  3 n \wmaxp,$
   satisfies
  $
  \sym((x_{0},t_{0}), \bnddom)
  \geq
  \frac{  1 }
  { 18 n^2 p \wmaxp }.
  $
\end{lemma}
Combining standard results on self-concordant barrier functions with a barrier for $p$-norms developed by Hertog et~al.\ \cite{den_hertog_sufficient_1995}, we can show the following properties of a function $\bndfn$ whose exact definition is given in Supplementary Material Section~\ref{sec:barfn}.
  \begin{corollary}
 \label{cor:barrierfn}
 The function $\bndfn$ is a self-concordant barrier for $\bnddom$ and it has complexity parameter $\theta(\bndfn) = O(m)$. 
 Its gradient $g_{\bndfn}$  is computable in $O(m)$ time, and an implicit representation of the Hessian $H_{\bndfn}$ can be computed in $O(m)$ time as well.
 \end{corollary}
The key reason we can use $\apxipm$ to give a fast algorithm for
Isotonic Regression is that we develop an efficient solver for
 linear equations in the Hessian of $\bndfn$.
The algorithm \hsol solves linear systems in Hessian matrices of the barrier function $\bndfn$.
The Hessian is composed of a structured main component plus a rank one matrix.
We develop a solver for the main component by doing a change of
variables to simplify its structure, and then factoring the matrix by
a block-wise $LDL^{\top}$-decomposition.  We can solve
straightforwardly in the $L$ and $L^{\top},$ and we show that the $D$
factor consists of blocks that are either diagonal or SDD, so we can
solve in this factor approximately using a nearly-linear time SDD
solver.
The algorithm \hsol is given in full in Supplementary Material Section~\ref{sec:solver},
along with a proof of the following result.
\begin{theorem}
For any instance of program~\eqref{eq:bndisoprog} given by some $(G,y)$,
at any point $z \in \bnddom$,
for any vector $a$,
$\hsol((G,y),z,\mu,a)$
returns a vector $b = M a$ for a symmetric linear operator $M$ satisfying
$
\nfrac{9}{10} \cdot H_{\bndfn}(z)^{-1} \preceq M \preceq
\nfrac{11}{10} \cdot H_{\bndfn}(z)^{-1}.
$
The algorithm fails with probability $< \mu$.
\hsol runs in time $\widetilde{O}(m \log n \log(1/\mu))$.
\label{thm:hessianapprox}
\end{theorem}
These are the ingredients we need to prove our main result on solving 
\piso. The algorithm \isoipm is simply \apxipm instantiated to solve
program~\eqref{eq:bndisoprog},
with an appropriate choice of parameters.
We state \isoipm informally as Algorithm~\ref{alg:isoipmSketch} below.
\isoipm is given in full as Algorithm~\ref{alg:isoipm} in Supplementary Material
  Section~\ref{sec:apxipm}.
\begin{proofof}{of Theorem~\ref{thm:isoipm}}
   \isoipm uses the symmetry lower bound $s =
  \frac{1}{18 n^2 p \wmaxp}$, the value upper bound $\valsup = 3 n
  \wmaxp$, and the error parameter $\eps = \frac{\delta}{\valsup}$
  when calling \apxipm.
  By Corollary~\ref{cor:barrierfn},
  the barrier function $\bndfn$ used by \isoipm has complexity parameter $\theta(\bndfn) \leq O(m)$.
  By Lemma~\ref{lem:startsym} the starting point $(x_{0},t_{0})$
  computed by \feasiblestart and used by \isoipm is feasible
  and has symmetry $\sym(x_{0}, \bnddom) \geq \frac{1}{18 n^2 p \wmaxp}$.

By Theorem~\ref{thm:pathalg} the point $(x_{\text{apx}},t_{\text{apx}})$ output by \isoipm satisfies
$
\frac{\pair{\vecw,t_{\text{apx}}}  - \opt}{\valsup - \opt} \leq \epsilon,
$
where \opt is the optimum of program~\eqref{eq:bndisoprog},
and $\valsup = 3n\wmaxp$ is the value used by \isoipm for the constraint $\pair{\vecw,t} \leq \valsup$,
  which is an upper bound on the supremum of objective values of feasible points of program~\eqref{eq:bndisoprog}.
 By Lemma~\ref{lem:boundeddom}, $\opt = \optiso$.
Hence,
$
\norm{y - x_{\text{apx}}}_{p}^{p}
\leq \pair{\vecw,t_{\text{apx}}}  \leq \opt + \epsilon \valsup  = \optiso + \delta.
$

  Again, by Theorem~\ref{thm:pathalg}, the number of calls to $\hsol$ by
  \isoipm is bounded by
  \begin{align*}
  O(T)
   \leq
   O\left( \sqrt{\theta(\bndfn)} \log\left( \nfrac{\theta(\bndfn)}{\epsilon \cdot s} \right) \right)
   \leq
   O\left( \sqrt{m} \log\left( \nfrac{n p \wmaxp}{\delta} \right) \right).
  \end{align*}
  Each call to $\hsol$ fails with probability $< n^{-3}.$ Thus, by a
  union bound, the probability that some call to $\hsol$ fails is upper bounded by
  $O(\sqrt{m} \log(np\wmaxp / \delta)) / n^3 = O(1/n)$.
The algorithm uses
$O\left( \sqrt{m} \log\left( \nfrac{n p \wmaxp}{\delta} \right)
\right)$
calls to $\hsol$ that each take time $\widetilde{O}(m \log^2 n)$, as
$\mu = n^3$.  Thus the total running time is
$\widetilde{O}\left(m^{1.5} \log^2 n \log\left( \nfrac{n p \wmaxp}{\delta}
  \right)\right)$.
\end{proofof}
\begin{mdframed}
\vspace{\algtopspace}
\captionof{table}{\label{alg:isoipmSketch} Sketch of Algorithm \isoipm \hfill \hfill}
\begin{tight_enumerate}
    \item
    Pick a starting point $(x,t)$ using the \feasiblestart algorithm
    \item
      \FOR $r = 1,2$
       \item \hspace{\pgmtab}
         \IF r = 1 \THEN $\gamma \assign -1;  \rho \assign 1; c = -$ gradient of $f$ at $(x,t)$
         \item \hspace{\pgmtab}
          \ELSE $\gamma
    \assign 1;  \rho \assign 1/\poly(n); c = (0,\vecw)$
    \item \hspace{\pgmtab} 
      \FOR $i \assign 1,\ldots, C_{1}m^{0.5} \log m  :$
    \item \hspace{\pgmtab} \hspace{\pgmtab}
      $\rho \assign \rho \cdot  (1 + \gamma C_{2}m^{-0.5})$
    \item \hspace{\pgmtab} \hspace{\pgmtab}
        Let $H,g$ be the Hessian and gradient of  $f_{c,\rho}$ at
        $x$
    \item \hspace{\pgmtab} \hspace{\pgmtab}
        Call \hsol to compute $z \approx H^{-1} g$
    \item \hspace{\pgmtab} \hspace{\pgmtab}
        Update $x \assign x - z$
    \item Return $x$.
    \end{tight_enumerate}
\end{mdframed}


\section{Algorithms for $\ell_{\infty}$ and Strict Isotonic Regression}
\label{sec:inf}

We now reduce $\ell_{\infty}$ Isotonic Regression and Strict Isotonic
Regression to the Lipschitz Learning problem, as defined
in~\cite{KyngRSS15}. Let $G=(V,E,\len)$ be any DAG with non-negative edge lengths
$\len: E \to \rea_{\geq 0}$, and $y: V \to \rea \cup \{ * \}$ a
partial labeling. We think of a partial labeling as a function that
assigns real values to a subset of the vertex set $V$.  We call such a
pair $(G,y)$ a \textbf{partially-labeled DAG}.  For a complete
labeling $x : V \to \rea,$ define the gradient on an edge
$(u,v) \in E$ due to $x$ to be
$\grad^{+}_{G}[x](u,v) \defeq \max \left \{
  \frac{x(u)-x(v)}{\len(u,v)}, 0 \right \}. $
If $\len(u,v) = 0$, then $\grad^{+}_{G}[x](u,v) =0$ unless
$x(u) > x(v)$, in which case it is defined as $+\infty$. Given a
partially-labelled DAG $(G,y)$, we say that a complete assignment $x$
is an \textbf{inf-minimizer} if it extends $y,$ and for all other
complete assignments $x'$ that extends $y$ we have
$$\max_{(u,v) \in E}\grad^{+}_{G}[x](u,v) \leq \max_{(u,v) \in E} \grad^{+}_{G}[x'](u,v) .$$ 
Note that when $\len = 0$, then
$\max_{(u,v) \in E}\grad^{+}_{G}[x](u,v) < \infty$ if and only if $x$
is isotonic on $G$.


Suppose we are interested in Isotonic Regression on a DAG $G(V,E)$
under $\norm{\cdot}_{w,\infty}$. To reduce this problem to that of
finding an inf-minimizer, we add some auxiliary nodes and edges to
$G$. Let $V_L,V_{R}$ be two copies of $V$. That is, for every vertex
$u \in V$,  add a vertex $u_L$ to $V_L$ and a vertex $u_R$ to
$V_R$. Let $E_L = \{ (u_L,u) \}_{u \in V}$ and
$E_R = \{ (u,u_R) \}_{u \in V}$. We then let
$\len'(u_L,u) = \nfrac{1}{w(u)}$ and $\len'(u,u_R) = \nfrac{1}{w(u)}$.
All other edge lengths are set to $0$. Finally, let
$G' = (V \cup V_L \cup V_R,E \cup E_L \cup E_R,\len').$ The partial
assignment $y'$ takes real values only on the the vertices in
$V_L \cup V_R.$ For all $u \in V$, $y'(u_L) := y(u)$,
$y'(u_R) := y(u)$ and $y'(u) := *.$ $(G^{\prime},y^{\prime})$ is our
partially-labeled DAG. Observe that $G'$ has $n'=3n$ vertices and $m'
= m+2n$ edges.
\begin{lemma}
\label{lem:infEq}
Given a DAG $G(V,E),$ a set of observations $y \in \rea^{V},$ and
weights $w,$ construct $G'$ and $y'$ as above. Let $x$ be an
inf-minimizer for the partially-labeled DAG $(G',y')$. Then,
$x \ |_{V}$ is the Isotonic Regression of $y$ with respect to $G$
under the norm $\norm{\cdot}_{w, \infty}.$
\end{lemma}
\vspace{-10pt}
\begin{proof}
  We note that since the vertices corresponding to $V$ in $(G',y')$
  are connected to each other by zero length edges,
  $\max_{(u,v) \in E}\grad^{+}_{G}[x](u,v) < \infty$ iff $x$ is
  isotonic on those edges. Since $G$ is a DAG, we know that there are
  isotonic labelings on G. When $x$ is isotonic on vertices
  corresponding to $V$, gradient is zero on all the edges going in
  between vertices in $V$. Also, note that every vertex $x$
  corresponding to $V$ in $G'$ is attached to two auxiliary nodes
  $x_L \in V_L, x_R \in V_R$. We also have $y'(x_L) = y'(x_R) = y(x).$
  Thus, for any $x$ that extends $y$ and is Isotonic on $G',$ the only
  non-zero entries in $\grad^{+}$ correspond to edges in $E_{R}$ and
  $E_{L},$ and thus
$$\max_{(u,v) \in E'}\grad^{+}_{G'}[x](u,v) =  \max_{u \in V} w_{u}
\cdot |y(u) - x(u)| = \norm{x-y}_{w,\infty}. \vspace{-15pt}
$$ 
\end{proof}
\vspace{-10pt}
Algorithm {\compinfmin} from~\cite{KyngRSS15} is proved to compute the
inf-minimizer, and is claimed to work for directed graphs (Section
5,~\cite{KyngRSS15}). We exploit the fact that Dijkstra's algorithm in
{\compinfmin} can be implemented in $O(m)$ time on DAGs using a
topological sorting of the vertices, giving a linear time algorithm
for computing the inf-minimizer. Combining it with the reduction given
by the lemma above, and observing that the size of $G'$ is $O(m+n),$
we obtain Theorem~\ref{thm:intro:inf}. A complete description of the
modified {\compinfmin} is given in Section~\ref{sec:algorithms}. We
remark that the solution to the $\ell_{\infty}$-Isotonic Regression that
we obtain has been referred to as {\sc Avg} $\ell_{\infty}$ Isotonic
Regression in the literature~\cite{StoutInfGeneral}. It is easy to
modify the algorithm to compute the {\sc Max, Min} $\ell_{\infty}$
Isotonic Regressions.  Details are given in Section~\ref{sec:appLex}.

For Strict Isotonic Regression, we define the lexicographic
ordering. Given $r \in \rea^m,$ let $\pi_r$ denote a permutation that
sorts $r$ in non-increasing order by absolute value, \emph{i.e.},
$\forall i \in [m-1], |r(\pi_r(i))| \ge |r(\pi_r(i+1))|.$
Given two vectors $r, s \in \rea^m,$ we write $r \lexless s$ to indicate
that $r$ is smaller than $s$ in the \textbf{lexicographic ordering} on
sorted absolute values,  \emph{i.e.}
\begin{align*}
  & \exists j \in [m],  \abs{r(\pi_{r}(j))} < \abs{s(\pi_{s}(j))} \text{ and } \forall i\in[j-1], \abs{r(\pi_{r}(i))} = \abs{s(\pi_{s}(i))} \\
  \text{ or } 
  & \forall i \in [m], \abs{r(\pi_{r}(i))} = \abs{s(\pi_{s}(i))}.
\end{align*}
Note that it is possible that 
  $r \lexless s$ and $s \lexless r$ while $r \neq s$. 
It is a
total relation: for every $r$ and $s$ at least one of
$r \lexless s$ or $s \lexless r$ is true.

Given a partially-labelled DAG $(G,y)$, we say that a complete
assignment $x$ is a \textbf{lex-minimizer} if it extends $y$ and for
all other complete assignments $x'$ that extend $y$ we have
$\grad^{+}_{G} [x] \lexless \grad^{+}_{G}
[x'].$ 
Stout~\cite{StoutStrict12} proves that computing the Strict Isotonic
Regression is equivalent to finding an Isotonic $x$ that minimizes
$z_{u} = w_{u}\cdot(x_{u}-y_{u})$ in the lexicographic ordering. With
the same reduction as above, it is immediate that this is equivalent
to minimizing $\grad^{+}_{G'}$ in the lex-ordering.
  \begin{lemma}
\label{lem:lexEq}
Given a DAG $G(V,E),$ a set of observations $y \in \rea^{V},$ and
weights $w,$   
construct $G'$ and $y'$ as above. Let $x$ be the lex-minimizer for the
partially-labeled DAG $(G',y')$. Then, $x \ |_{V}$ is the Strict Isotonic
Regression of $y$ with respect to $G$ with weights $w.$
\end{lemma}

As for inf-minimization, we give a modification of the algorithm
{\complexmin} from~\cite{KyngRSS15} that computes the lex-minimizer in
$O(mn)$ time. The algorithm is described in
Section~\ref{sec:algorithms}. Combining this algorithm with the reduction from
Lemma~\ref{lem:lexEq}, we can compute the Strict Isotonic Regression
in $O(m'n') = O(mn)$ time, thus proving Theorem~\ref{thm:intro:strict}.


\smallskip {\bf Acknowledgements.} 
We thank Sabyasachi Chatterjee for
introducing the problem to us, and Daniel Spielman for his advice,
comments, and discussions. We also thank Quentin Stout for his
suggestions on a previous version of the manuscript, and anonymous
reviewers for their comments.
\renewcommand\refname{{References}} \bibliographystyle{unsrt}
{\small\bibliography{papers}}

\begin{thebibliography}{10}

\bibitem{Ayer55}
M.~Ayer, H.~D. Brunk, G.~M. Ewing, W.~T. Reid, and E.~Silverman.
\newblock An empirical distribution function for sampling with incomplete
  information.
\newblock {\em The Annals of Mathematical Statistics}, 26(4):pp. 641--647,
  1955.

\bibitem{Barlow72}
D.~J. Barlow, R. E .and~Bartholomew, J.M. Bremner, and H.~D. Brunk.
\newblock {\em Statistical inference under order restrictions: the theory and
  application of {Isotonic} {Regression}}.
\newblock Wiley New York, 1972.

\bibitem{Gebhardt70}
F.~Gebhardt.
\newblock An algorithm for monotone {Regression} with one or more independent
  variables.
\newblock {\em Biometrika}, 57(2):263--271, 1970.

\bibitem{OR1}
W.~L. Maxwell and J.~A. Muckstadt.
\newblock Establishing consistent and realistic reorder intervals in
  production-distribution systems.
\newblock {\em Operations Research}, 33(6):1316--1341, 1985.

\bibitem{OR2}
R.~Roundy.
\newblock A 98\%-effective lot-sizing rule for a multi-product, multi-stage
  production / inventory system.
\newblock {\em Mathematics of Operations Research}, 11(4):pp. 699--727, 1986.

\bibitem{Sig1}
S.T. Acton and A.C. Bovik.
\newblock Nonlinear image estimation using piecewise and local image models.
\newblock {\em Image Processing, IEEE Transactions on}, 7(7):979--991, Jul
  1998.

\bibitem{Stats1}
C.I.C. Lee.
\newblock The {Min-Max} algorithm and {{Isotonic}} {Regression}.
\newblock {\em The Annals of Statistics}, 11(2):pp. 467--477, 1983.

\bibitem{Stats2}
R.~L. Dykstra and T.~Robertson.
\newblock An algorithm for {Isotonic} {Regression} for two or more independent
  variables.
\newblock {\em The Annals of Statistics}, 10(3):pp. 708--716, 1982.

\bibitem{ChatterjeeGS15}
S.~{Chatterjee}, A.~{Guntuboyina}, and B.~{Sen}.
\newblock {On Risk Bounds in {Isotonic} and Other Shape Restricted {Regression}
  Problems}.
\newblock {\em The Annals of Statistics}, to appear.

\bibitem{KalaiS09}
A.~T. Kalai and R.~Sastry.
\newblock The isotron algorithm: High-dimensional {Isotonic} {Regression}.
\newblock In {\em COLT}, 2009.

\bibitem{Learning2}
T.~Moon, A.~Smola, Y.~Chang, and Z.~Zheng.
\newblock Intervalrank: {Isotonic} {Regression} with listwise and pairwise
  constraints.
\newblock In {\em WSDM}, pages 151--160. ACM, 2010.

\bibitem{Learning4}
S.~M Kakade, V.~Kanade, O.~Shamir, and A.~Kalai.
\newblock Efficient learning of generalized linear and single index models with
  {Isotonic} {Regression}.
\newblock In {\em NIPS}. 2011.

\bibitem{ZadroznyE02}
B.~Zadrozny and C.~Elkan.
\newblock Transforming classifier scores into accurate multiclass probability
  estimates.
\newblock KDD, pages 694--699, 2002.

\bibitem{NarasimhanA13}
H.~Narasimhan and S.~Agarwal.
\newblock On the relationship between binary classification, bipartite ranking,
  and binary class probability estimation.
\newblock In {\em NIPS}. 2013.

\bibitem{AngelovHKW06}
S.~Angelov, B.~Harb, S.~Kannan, and L.~Wang.
\newblock Weighted {Isotonic} {Regression} under the $l_1$ norm.
\newblock In {\em SODA}, 2006.

\bibitem{Learning3}
K.~Punera and J.~Ghosh.
\newblock Enhanced hierarchical classification via {Isotonic} smoothing.
\newblock In {\em WWW}, 2008.

\bibitem{Web2}
Z.~Zheng, H.~Zha, and G.~Sun.
\newblock Query-level learning to rank using {Isotonic} {Regression}.
\newblock In {\em Communication, Control, and Computing, Allerton Conference
  on}, 2008.

\bibitem{HochbaumQ03}
D.S. Hochbaum and M.~Queyranne.
\newblock Minimizing a convex cost closure set.
\newblock {\em SIAM Journal on Discrete Mathematics}, 16(2):192--207, 2003.

\bibitem{StoutIsotonicPartitioning13}
Q.~F. Stout.
\newblock {Isotonic} {Regression} via partitioning.
\newblock {\em Algorithmica}, 66(1):93--112, 2013.

\bibitem{StoutInfGeneral}
Q.~F. Stout.
\newblock Weighted $l_\infty$ {Isotonic} {Regression}.
\newblock {\em Manuscript}, 2011.

\bibitem{StoutStrict12}
Q.~F. Stout.
\newblock {S}trict $l_{\infty}$ {I}sotonic {R}egression.
\newblock {\em Journal of Optimization Theory and Applications},
  152(1):121--135, 2012.

\bibitem{StoutSurvey}
Q.~F Stout.
\newblock Fastest {Isotonic} {Regression} algorithms.
\newblock \url{http://web.eecs.umich.edu/~qstout/IsoRegAlg_140812.pdf}.

\bibitem{StoutMulti15}
Q.~F. Stout.
\newblock {Isotonic} {Regression} for multiple independent variables.
\newblock {\em Algorithmica}, 71(2):450--470, 2015.

\bibitem{KaufmanT93}
Y.~Kaufman and A.~Tamir.
\newblock Locating service centers with precedence constraints.
\newblock {\em Discrete Applied Mathematics}, 47(3):251 -- 261, 1993.

\bibitem{StoutInfSpecial15}
Q.~F. Stout.
\newblock L infinity {Isotonic} {Regression} for linear, multidimensional, and
  tree orders.
\newblock {\em CoRR}, 2015.

\bibitem{DaitchS08}
S.~I. Daitch and D.~A. Spielman.
\newblock Faster approximate lossy generalized flow via interior point
  algorithms.
\newblock STOC '08, pages 451--460. ACM, 2008.

\bibitem{Madry13}
A.~Madry.
\newblock Navigating central path with electrical flows.
\newblock In {\em FOCS}, 2013.

\bibitem{LeeS14}
Y.~T. Lee and A.~Sidford.
\newblock Path finding methods for linear programming: Solving linear programs
  in $\tilde{O}(\sqrt{rank})$ iterations and faster algorithms for maximum
  flow.
\newblock In {\em FOCS}, 2014.

\bibitem{SpielmanT04}
D.~A. Spielman and S.~Teng.
\newblock Nearly-linear time algorithms for graph partitioning, graph
  sparsification, and solving linear systems.
\newblock STOC '04, pages 81--90. ACM, 2004.

\bibitem{KoutisMP11}
I.~Koutis, G.~L. Miller, and R.~Peng.
\newblock A nearly-{$m \log n$} time solver for {SDD} linear systems.
\newblock FOCS '11, pages 590--598, Washington, DC, USA, 2011. IEEE Computer
  Society.

\bibitem{CohenSTOC}
M.~B. Cohen, R.~Kyng, G.~L. Miller, J.~W. Pachocki, R.~Peng, A.~B. Rao, and
  S.~C. Xu.
\newblock Solving {SDD} linear systems in nearly {$m \log^{{1}/{2}} n$} time.
\newblock STOC '14, 2014.

\bibitem{KyngRSS15}
R.~Kyng, A.~Rao, S.~Sachdeva, and D.~A. Spielman.
\newblock Algorithms for {Lipschitz} learning on graphs.
\newblock In {\em Proceedings of {COLT} 2015}, pages 1190--1223, 2015.

\bibitem{BoydV04}
S.~Boyd and L.~Vandenberghe.
\newblock {\em Convex Optimization}.
\newblock Cambridge University Press, 2004.

\bibitem{den_hertog_sufficient_1995}
D.~den Hertog, F.~Jarre, C.~Roos, and T.~Terlaky.
\newblock A sufficient condition for self-concordance.
\newblock {\em Math. Program.}, 69(1):75{\textendash}88, July 1995.

\bibitem{renegar_mathematical_2001}
J.~Renegar.
\newblock {\em A mathematical view of interior-point methods in convex
  optimization}.
\newblock SIAM, 2001.

\bibitem{nemirovski_lecture_2004}
A.~Nemirovski.
\newblock Lecure notes: Interior point polynomial time methods in convex
  programming, 2004.

\bibitem{McShane}
E.~J. McShane.
\newblock Extension of range of functions.
\newblock {\em Bull. Amer. Math. Soc.}, 40(12):837--842, 12 1934.

\bibitem{Whitney}
H.~Whitney.
\newblock Analytic extensions of differentiable functions defined in closed
  sets.
\newblock {\em Transactions of the American Mathematical Society}, 36(1):pp.
  63--89, 1934.

\end{thebibliography}
\newpage
\appendix
\section{IPM Definitions and Proofs}
\label{sec:ipmsuppl}

\subsection{Definitions}
\label{sec:ipmdefs}
\Snote{TODO: do we also need to define spectral order?}

Given a positive definite  $n \times n$ matrix $A$, we define the norm $\norm{ \cdot }_{A}$ by
\[
  \norm{ x }_{A} \defeq \sqrt{ x^T A x }.
\]
Given a twice differentiable function $f$ with domain $D_{f}$, which has positive definite Hessian $H(x)$ at some $x \in D_{f}$, we define
\[
\norm{ y }_{x} \defeq  \norm{ y }_{H(x)},
\]
and when $M$ is a matrix, let $\norm{M}_{x}$ denote the corresponding induced matrix norm.

We let $B_{x}(y,r)$ denote the open ball centered at $y$ of radius $r$ in the $\norm{ \cdot }_{x}$ norm.

Again, suppose $f$ is a twice differentiable convex function with Hessian $H$,
defined on a domain $D_{f}$.
If for all $x \in D_{f}$ we have
\[
B_{x}(x,1) \subseteq D_{f},
\]
and for all $y \in B_{x}(x,1)$ and all $v \neq 0$ we have 
\[
1 - \norm{y - x}_{x}
\leq
\frac{\norm{v}_{y}}{\norm{v}_{x}}
\leq \frac{1}{1 - \norm{y - x}_{x}}, 
\]
then we say the function is self-concordant. We denote the set of self-concordant functions by $\mathcal{SC}$.
A key theorem about self-concordant functions is the following (Theorem 2.2.1 of Renegar~\cite{renegar_mathematical_2001}).
\begin{theorem}
Suppose $f$ is a twice differentiable function with Hessian $H$,
defined on a domain $D_{f}$,
and for all $x \in D_{f}$ we have
\[
B_{x}(x,1) \subseteq D_{f},
\]
Then $f \in \mathcal{SC}$
iff
\[
\norm{H(x)^{-1}H(y)}_{x}, \norm{H(x)^{-1}H(y)}_{x} \leq \frac{1}{(1 - \norm{y - x}_{x})^{2}}.
\]
Also $f \in \mathcal{SC}$
iff
\[
\norm{I - H(x)^{-1}H(y)}_{x}, \norm{I - H(x)^{-1}H(y)}_{x} \leq \frac{1}{(1 - \norm{y - x}_{x})^{2}} - 1.
\]
\label{thm:schess}
\end{theorem}

If $f \in SC$ also satisfies $\sup_{x \in D_{f}} \norm{g_{x}(x)}_{x}^{2} < \infty$,
we say that $f$ is a self-concordant barrier function.
Given any $\theta(f) \geq \sup_{x \in D_{f}} \norm{g_{x}(x)}_{x}^{2}$, we say $\theta(f)$ is a complexity parameter for $f$.
We denote the set of self-concordant barrier functions by $\mathcal{SCB}$. 

We need the following notion of symmetry. We state a definition that
is equivalent to the definition used by Renegar (Section 2.3.4 of
\cite{renegar_mathematical_2001}).
\begin{definition}
Given a convex set $S$ and a point $x \in S$, the symmetry of $x$ w.r.t. $S$ is defined as
\[
\sym(x,S) \defeq  \inf_{z \in \partial S} \inf \left\{ t > 0 : x + \frac{ (x - z)}{t} \in S \right\}.
\]
\end{definition}

\subsection{A Barrier Function for $\bnddom$}
\label{sec:barfn}
Hertog et~al.\ \cite{den_hertog_sufficient_1995} proved the existence
of self-concordant barrier functions for a class of domains including
ones capable of expressing program~\eqref{eq:isoprog}. The exact
statement we wish to employ can be found in lecture notes by
Nemirovski \cite{nemirovski_lecture_2004}.
\begin{theorem}
For every pair of variables $(x,t) \in \Reals{2}$, and for every constant $p \geq 1$,
  a self-concordant barrier function $f \in \mathcal{SCB}$ exists for the domain
\[
\{ (x,t) \in \Reals{2} : \abs{x}^p \leq t \}.
\]
This barrier function is given by
\[
f(t,x) \defeq - \log( t^{2/p} - x^2 ) - 2\log t,
\]
and has complexity parameter $\theta(f) \leq 4 $.
\label{thm:pbarrier}
\end{theorem}

We are now ready to introduce a number of barrier functions:
\begin{equation}
\begin{aligned}
\odmfn(x,t) & \defeq 
  \left( \sum_{ v \in V } - \log\left( t(v)^{2/p} - (x(v) - y(v))^2\right)  - 2\log t(v) \right)
  + \left( \sum_{(a,b) \in E} -\log(x(b)-x(a)) \right). \\
 \vbfn(x,t) & \defeq  - \log(\valsup - \pair{\vecw,t}). \\
\bndfn(x,t) & \defeq  \odmfn(x,t) + \vbfn(x,t). \\
\end{aligned}
\end{equation}

\begin{proofof}{of Corollary~\ref{cor:barrierfn}}
To prove the corollary, we need the standard fact that
$-\log x$ is a self-concordant barrier for the domain $x \geq 0$ with complexity parameter $1$,
as shown in Renegar's section 2.3.1 \cite{renegar_mathematical_2001}.
We also need standard results on composition of barrier functions (adding barriers and composition with an affine function),
as given by Renegar's Theorems 2.2.6, 2.2.7, 2.3.1, and 2.3.2 \cite{renegar_mathematical_2001}.
Given these and Theorem~\ref{thm:pbarrier}, the corollary follows immediately.
\end{proofof}

\subsection{Fast Solver for Approximate Hessian Inverse}
\label{sec:solver}

\begin{mdframed}
\vspace{\algtopspace}
\captionof{table}{
\label{alg:hessiansol}
Algorithm
$\hsol\left((G,y),(x,t),\mu,a\right)$:
Given a \piso instance $(G,y)$,
a feasible point $(x,t)$ of program,
a vector $a$,
outputs vector $b$.\hfill \hfill}
    \begin{tight_enumerate}
	\item $u \assign \frac{1}{(K-\pair{\vecone,t})} \vecone$.
	\item $\tau \assign 1/50$
	\item $M \assign \blksol\left( (G,y), (x,t), \mu, \tau \right)$
	\item return
	$\rankone
	\left(
		M, u, a
	\right) $
      \end{tight_enumerate}
\end{mdframed}

\begin{mdframed}
\vspace{\algtopspace}
\captionof{table}{
\label{alg:blocksol}
Algorithm $\blksol\left((G,y),(x,t),\mu,\tau\right)$\hfill \hfill}
    \begin{tight_enumerate}
      \item Let $ r \assign B^{T} (x \oplus y)$.
      \item For each $v \in V$, identify $ t(\hat{v},v) = t(v) $.
      \item Compute $R$, $T$ and $C$ as given by equations \eqref{eq:hessR}, \eqref{eq:hessT}, and \eqref{eq:hessC}.
      \item
      $S \assign Q^TB(R - CT^{-1}C^T)B^TQ.$
      \item
       $M_{S} \assign \sddsol(S,\mu,\tau)$.
      \item
      $ Z \assign \begin{bmatrix}
  	I & 0\\
	-Q^TCBT^{-1} & I
       \end{bmatrix}.
       $
       \item
       Return a procedure that given vector $a$ returns vector
       \[
      	b \assign Z^T
	\begin{bmatrix}
 	 T^{-1} & 0\\
 	 0 & M_{S}
 	\end{bmatrix}
	 Z a.
       \] 
      \end{tight_enumerate}      
\end{mdframed}

\begin{mdframed}
\vspace{\algtopspace}
\captionof{table}{
\label{alg:rankone}
Algorithm
$\rankone(M, u , a)$:
Given
a linear operator $M$,
a vector $u$,
and a vector $a$,
outputs vector $b$.\hfill \hfill}
    \begin{tight_enumerate}
      \item $w = M u$.
      \item $z = M a$.
      \item Return
      \[
      b = z
      - \frac{w^T a}{1+u^Tw}  w .
      \]
      \end{tight_enumerate}
\end{mdframed}

We introduce an extended graph $\widehat{G}= (V \union \widehat{V},E \union \widehat{E})$, which includes our original vertex set $V$, and a second vertex set
\[
\widehat{V} = \setof{ \hat{v} : v \in V }.
\]
We define an additional set of edges
\[
\widehat{E} =  \setof{ (\hat{v},v) : v \in V}
\]
Given vectors $t \in \Reals{\widehat{E}}$ and $r \in \Reals{E \union \widehat{E}}$, we define a function
 \[
 h(r,t) \defeq
 \left( \sum_{e \in \widehat{E} } - \log( t(e)^{2/p} - r(e)^2)  - 2\log t(e) \right)
 + \left( \sum_{e \in E} -\log(r(e)) \right).
 \]
We associate with $\widehat{G}$ a matrix $B$ known as the signed edge-vertex incidence matrix. $B$ has rows indexed by the set $V \union \widehat{V}$, and columns indexed by the set $E \union \widehat{E}$.
It is given by
\[
B(a,e) =
\begin{cases}
1  & \text{if } e = (a,b) \in E \union \widehat{E} \text{ for some } b \in V \union \widehat{V} \\
-1 & \text{if } e = (a,b) \in E \union \widehat{E} \text{ for some } b \in V \union \widehat{V} \\
0  & \text{otherwise}.
\end{cases}
\]
Now, we define a vector $x \oplus y \in \Reals{V \union \widehat{V}} $by 
\[
(x \oplus y)(u) =
\begin{cases}
x(u) & \text{for } u \in V\\
y(v) & \text{where } \hat{v} = u  \text{ and } u \in \widehat{V}
\end{cases}
\]
Note that $\sizeof{\widehat{E}} = \sizeof{V}$. Abusing notation, we identify the vector $t \in \Reals{\widehat{E}}$ with the vector $t \in \Reals{V}$ by equating $t(\hat{v},v) = t(v) $. We then get
\[
F(x,t) = h(B^{T} (x \oplus y),t).
\]
We compute the Hessian $H_{h}$ of $h(r,t)$ in variables $r$ and $t$. The Hessian can be expressed as a block matrix
\[
H_{h} = 
\begin{bmatrix}
  T & C^T\\
  C & R
 \end{bmatrix},
\]
where $T$ contains derivatives in two coordinates of $t$, while $R$ contains derivatives in two coordinates in $r$, and $C$ has the cross-terms. $T$ and $R$ are square diagonal matrices, and $C$ is not generally square, but has non-zero entries on the principal diagonal. In fact, the only thing we will need about the explicit forms of these matrices is that they are efficiently computable. For completeness, we state them:
\begin{equation}
T(e,e) =  \left( \frac{\frac{2}{p} t(e)^{-1+2/p}}{t(e)^{2/p}-r(e)^2}\right)^2
 - \left( \frac{2}{p}-1 \right) \left( \frac{2}{p} \right) \frac{ t(e)^{-2+2/p} }{t(e)^{2/p}-r(e)^2} 
 + \frac{2}{t(e)^2}, \text{ where } e \in \widehat{E}
 \label{eq:hessT}
\end{equation}
and
\begin{equation}
R(e,e) = 
\begin{cases}
\left( \frac{2r(e)}{t(e)^{2/p}-r(e)^2} \right)^2 + \frac{2}{t(e)^{2/p}-r(e)^2}
& \text{ for } e \in \widehat{E}
\\
1/r(e)^2
& \text{ for } e \in E,
\end{cases}
 \label{eq:hessR}
\end{equation}
while
\begin{equation}
C(e,e) = - \frac{4}{p}\frac{t(e)^{-1+2/p}r(e)}{(t(e)^{2/p}-r(e)^2)^2} \text{ where } e \in \widehat{E}.
\label{eq:hessC}
\end{equation}

Finally, let $Q$ denote the $\abs{V \union \widehat{V}} \times \abs{V}$ projection matrix which maps $x$ to $(x \oplus 0)$. It is a matrix with non-zeroes only on the principal diagonal:
\[
Q(v,v) =
\begin{cases}
1  & \text{for } v \in V \\
0  & \text{otherwise}.
\end{cases}
\] 
To prove Theorem~\ref{thm:hessianapprox}, we will need three results:
The first is a theorem on fast SDD solvers proven by Koutis et al.~\cite{KoutisMP11}.
\begin{theorem}
Given an $n \times n$ SDD matrix $X$ with m non-zero entries,
an error probability $\mu$,
and an error parameter $\tau$,
with probability $\geq 1 - \mu$
the procedure $\sddsol(X,\mu,\tau)$
returns an (implicitly represented) symmetric linear operator $M$ satisfying
 \[
(1-\tau) X^{-1} \preceq M \preceq (1+\tau)  X^{-1}.
\]
 $\sddsol(X,\mu,\tau)$ runs in time  $\widetilde{O}(m \log n \log(1/\mu) \log(1/\tau))$, and $M$ can be applied to a vector in time $\widetilde{O}(m \log n \log(1/\mu) \log(1/\tau))$ as well.
 \label{thm:sddsol}
\end{theorem}

\begin{lemma}
Suppose $X$ is a positive definite matrix,
and $\tau \in [0,1/5)$ is an error parameter,
and we are given a symmetric linear operator $M$ satisfying
\[
(1-\tau) X^{-1} \preceq M \preceq (1+\tau)  X^{-1},
\]
and suppose we are given a vector $u \in \Reals{n}$.
Then $\rankone(M,u,a)$
acts as a linear operator on $a$
and returns a vector $b = Z a$
for a symmetric matrix $Z$ satisfying
\[
(1-5\tau) (X+uu^T)^{-1} \preceq Z \preceq (1+5 \tau)  (X+uu^T)^{-1}.
\]
If $M$ can be applied in time $T_{M}$,
then  \rankone runs in time $O(T_{M} + n)$.
\label{lem:rank1sol}
\end{lemma}
\begin{lemma}
For any instance of program~\eqref{eq:bndisoprog} given by some $(G,y)$,
at any point $z \in \bnddom$,
given an error probability $\mu$,
and an error parameter $\tau$,
with probability $\geq 1 - \mu$
the procedure $\blksol(X,\mu,\tau)$
returns an (implicitly represented) symmetric linear operator $M$ satisfying
\[
(1-\tau) H_{\odmfn}(z)^{-1} \preceq M \preceq(1+\tau) H_{\odmfn}(z)^{-1}.
\]
 $\blksol(X,\mu,\tau)$ runs in time  $\widetilde{O}(m \log n \log(1/\mu) \log(1/\tau))$, and $M$ can be applied 
 to a vector in time\\ $\widetilde{O}(m \log n \log(1/\mu) \log(1/\tau))$ as well.
\label{lem:blksol}
\end{lemma}
We prove Lemmas~\ref{lem:blksol} and~\ref{lem:rank1sol} at the end of this section.

\begin{proofof}{of Theorem~\ref{thm:hessianapprox}}
 By Lemma~\ref{lem:blksol},
 $\blksol\left( (G,y), (x,t),\mu,1/50 \right)$
 returns an implicitly represented linear operator $M$ satisfying
\[
\left(1-\frac{1}{50} \right) H_{\odmfn}((x,t))^{-1} \preceq M \preceq \left(1+\frac{1}{50}\right) H_{\odmfn}((x,t))^{-1}.
\]
This $M$ satisfies the requirements of $M$ in Lemma~\ref{lem:rank1sol}
with $X = H_{\odmfn}((x,t))$ and $\tau = 1/50$.
With $u = \frac{1}{(K-\pair{\vecone,t})} \vecone$,
where $H_{\odmfn}(x,t) + uu^T = H_{\bndfn}(x,t)$,
we get that
	$\rankone
	\left(
		M, u, a
	\right) $
returns a vector $Z a$,
for a symmetric matrix $Z$ satisfying
\[
\frac{9}{10} H_{\bndfn}(x,t)^{-1} \preceq Z \preceq \frac{11}{10} H_{\bndfn}(x,t)^{-1}.
\]
The total running time is $\widetilde{O}(m \log n \log(1/\mu))$, as the running time of \blksol dominates. The algorithms fails only if \blksol fails, which happens with probability $< \mu$.
\end{proofof}

\begin{proofof}{of Lemma~\ref{lem:blksol}}
Note $T$ is a diagonal matrix, so that its inverse can be computed in linear time.

Using standard facts about the Hessian under function composition,
we can express the Hessian of $\odmfn$ as
 \begin{equation*}
H_{\odmfn} = 
 \begin{bmatrix}
  I & 0\\
  0 & Q^TB
 \end{bmatrix}
\begin{bmatrix}
  T & C^T\\
  C & R
 \end{bmatrix}
 \begin{bmatrix}
  I & 0\\
  0 & B^TQ
 \end{bmatrix}
 =
 \begin{bmatrix}
  T & C^TB^TQ\\
  Q^TBC &   Q^TBRB^TQ.
 \end{bmatrix}
\end{equation*}

A block-wise LDU decomposition of $H_{\odmfn}$ gives

\begin{equation*}
H_{\odmfn}
 = 
 \begin{bmatrix}
  I & 0\\
Q^TBCT^{-1} & I
 \end{bmatrix}
\begin{bmatrix}
  T & 0\\
  0 & S
 \end{bmatrix}
 \begin{bmatrix}
  I & T^{-1}C^TB^TQ\\
  0 & I
 \end{bmatrix}.
\end{equation*}
Where the matrix
\[
S \defeq Q^TBRB^TQ - Q^TBCT^{-1}C^TB^TQ = Q^TB(R - CT^{-1}C^T)B^TQ
\]
is the \emph{Schur-complement} of $T$ in $H_{\bar{f}}$.
Now, $R - CT^{-1}C^T$ is the Schur-complement of $T$ in $H$.
A standard result about Schur complements states that
$H$ is positive definite if and only if both $T$ and $R - CT^{-1}C^T$ are positive definite.
We know that $H$ is positive definite, and consequently $R - CT^{-1}C^T$ is too.
But $R - CT^{-1}C^T$ is a diagonal matrix, and so every entry must be strictly positive.
This implies that $B(R - CT^{-1}C^T)B^T$ is a Laplacian matrix. \Rnote{Elaborate on prev?} 
The matrix has $O(m)$ non-zero entries. 
Since $S = Q^TB(R - CT^{-1}C^T)B^TQ$ is a principal minor of a Laplacian matrix,
it is positive definite and SDD.
Because $S$ is PD and SDD, 
by Theorem~\ref{thm:sddsol},
using \sddsol
we can compute an (implicitly represented) approximate inverse matrix $M_{S}$ that satisfies
\begin{equation}
(1-\tau)S^{-1} \preceq M_{S}\preceq (1+\tau)S^{-1}.
\label{eq:schurapprox}
\end{equation}
in time $\widetilde{O}(m\log n \log \frac{1}{\mu} \log \frac{1}{\tau} )$.
This call may fail with a probability $<\mu$.
The matrix $M_{S}$ can be applied in time $\widetilde{O}(m\log n \log \frac{1}{\mu} \log \frac{1}{\tau} )$.

A block-wise inversion of the Hessian gives
\begin{align}
H_{\odmfn}^{-1}
& = 
 \begin{bmatrix}
  I & -T^{-1}C^TB^TQ\\
  0 & I
 \end{bmatrix}
\begin{bmatrix}
  T^{-1} & 0\\
  0 & S^{-1}
 \end{bmatrix}
 \begin{bmatrix}
  I & 0\\
-Q^TCBT^{-1} & I
 \end{bmatrix}.
 \label{eq:hessianblockinv}
\end{align}
We define
\begin{align}
M
& = 
 \begin{bmatrix}
  I & -T^{-1}C^TB^TQ\\
  0 & I
 \end{bmatrix}
\begin{bmatrix}
  T^{-1} & 0\\
  0 & M_{S}
 \end{bmatrix}
 \begin{bmatrix}
  I & 0\\
-Q^TCBT^{-1} & I
 \end{bmatrix}.
\end{align}
By equations~\eqref{eq:schurapprox} and \eqref{eq:hessianblockinv}, and the fact that for all matrices $W$, $X \preceq Y$
implies $ W X W^T \preceq W Y W^T$, it follows that
\begin{equation*}
(1-\tau)H_{\odmfn}^{-1}\preceq M \preceq (1+\tau)H_{\odmfn}^{-1}.
\end{equation*}

We observe that the output of $\blksol\left((G,y),(x,t),\mu,\tau\right) $ is a procedure which applies $M$. We require a constant number of linear time matrix operations (inversion of a diagonal matrix, multiplication of a vector with matrix), and one call to \sddsol, which runs in time $\widetilde{O}(m\log n \log \frac{1}{\mu} \log \frac{1}{\tau} )$. This call dominates the running time of \blksol. The call to \sddsol may fail with a probability $<\mu$, and in that case \blksol also fails.
\end{proofof}

\begin{proofof}{of Lemma~\ref{lem:rank1sol}}
From our assumptions about $M$ and the computation in \rankone, it follows that
\[
b = Z a .
\]
for some
\[
Z = M 
- \frac{M u u^T M }{1+u^T M u^T },
\]
where $\tau = \frac{\delta}{5} < 1/5$ and 
\[
(1-\tau) X^{-1} \preceq M \preceq (1+\tau)  X^{-1}.
\]
Thus, \rankone acts as a linear operator on $a$, and it is symmetric.
Suppose $Y = X + uu^T$, then by the Sherman-Morrison formula,
\[
Y^{-1} = X^{-1}
-
\frac{X^{-1} u u^T X^{-1} }
{1+ u^T X^{-1} u}.
\]
We can restate the spectral inequalities for $M$ as $M =X^{-1} + E$, for some symmetric matrix $E$ with
\[
-\tau X^{-1} \preceq E \preceq \tau X^{-1}.
\]
%
We want to show that
\[
(1 - \delta) Y^{-1} \preceq Z \preceq (1+\delta) Y^{-1},
\]
where $\delta = 5 \tau$.

First observe that for any vector $v$,
\[
v^T Y^{-1} v = v^T X^{-1} v
-
\frac{v^T X^{-1} u u^T X^{-1} v}
{1+ u^T X^{-1} u}
=
\frac{ v^T X^{-1} v }
{1+ u^T X^{-1} u}
+
\frac{  (v^T X^{-1} v) (u^T X^{-1} u)  - (u^T X^{-1} v)^2 }
{1+ u^T X^{-1} u},
\]
where in the latter expression, both terms are non-negative.
Similarly
\[
v^T Z v = v^T M v
-
\frac{v^T M u u^T M v}
{1+ u^T M u}
=
\frac{ v^T M v }
{1+ u^T M u}
+
\frac{  (v^T M v) (u^T M u)  - (u^T M v)^2 }
{1+ u^T M u},
\]
and again in the final expression, both terms are non-negative.
We state two claims that help prove the main lemma.
\begin{claim}
\[
\abs{
\frac{1}
{1 + u^T M u }
-
\frac{1}
{1+u^T X^{-1} u}
}
\leq
\frac{\tau }
{ 1- \tau }
\cdot 
\frac{1}
{1+u^T X^{-1} u}.
\]
\label{clm:denomapx}
\end{claim}
\begin{claim}
\begin{align*}
\abs{
(v^T X^{-1} v) (u^T X^{-1} u) - (v^T X^{-1} u)^2
-\left((v^T M v) (u^T M u) - (v^T M u)^2\right)
} \\
\leq
2(\tau + \tau^2) \left( (v^T X^{-1} v) (u^T X^{-1} u) - (v^T X^{-1} u)^2 \right) .
\end{align*}
\label{clm:crossapx}
\end{claim}
We also make frequent use of the fact that $1 + u^T M u \geq 1 + (1-\tau) u^T X^{-1} u \geq (1-\tau) (1 + u^T X^{-1} u)$.
Thus

\begin{align*}
\abs{ v^T Z v - v^T Y^{-1} v }
& \leq
\abs{
\frac{ v^T M v - v^T X^{-1} v}
{1+ u^T M u}
}
+
v^T X^{-1} v
\cdot
\abs{
\frac{1}
{1 + u^T M u }
-
\frac{1}
{1+u^T X^{-1} u}
}
\\
& +
\abs{
\frac{  (v^T M v) (u^T M u)  - (u^T M v)^2 
-
 (v^T X^{-1} v) (u^T X^{-1} u)  - (u^T X^{-1} v)^2 }
{1+ u^T M u}
}
\\
& +
\left(  (v^T X^{-1} v) (u^T X^{-1} u)  - (u^T X^{-1} v)^2 \right)
\abs{
\frac{1}
{1 + u^T M u }
-
\frac{1}
{1+u^T A^{-1} u}
}
\\
& \leq
\frac{2 \tau }
{ 1- \tau }
\cdot
\frac{ v^T X^{-1} v }
{1+ u^T X^{-1} u}
+
\frac{3 \tau + 2\tau^2}
{ 1- \tau }
\cdot
\frac{ (v^T X^{-1} v) (u^T X^{-1} u)  - (u^T X^{-1} v)^2 }
{1+ u^T X^{-1} u}
\\
& \leq
\frac{3 \tau + 2\tau^2}
{ 1- \tau }
v^T Y^{-1} v.
\\
& \leq
5 \tau 
\cdot
v^T Y^{-1} v.
\end{align*}
\end{proofof}

\begin{proofof}{of Claim~\ref{clm:denomapx}}
\begin{align*}
\abs{
\frac{1}
{1 + u^T M u }
-
\frac{1}
{1+u^T X^{-1} u}
}
&= 
\abs{
\frac{u^T E u}
{(1 + u^T M u)(1+u^T X^{-1} u) }
}
\\
& \leq
\frac{1}
{1 + u^T M u }
\cdot 
\frac{\tau u^T X^{-1} u}
{1+u^T X^{-1} u}
\\
& \leq
\frac{\tau }
{ 1- \tau }
\cdot 
\frac{1}
{1+u^T X^{-1} u}.
\end{align*}
\end{proofof}

\begin{proofof}{of Claim~\ref{clm:crossapx}}
Let
\begin{align*}
v = \alpha \hat{v} \text{ where } \hat{v} X^{-1} \hat{v} = 1, \\
u = \beta \hat{u}  \text{ where } \hat{u} X^{-1} \hat{u} = 1.
\end{align*}
Also let $\hat{u} = \gamma \hat{v} + \sqrt{1 - \gamma^2} \hat{w}$, where $\hat{w} X^{-1} \hat{v} = 0$.
Now
\[
1 = \hat{u} X^{-1} \hat{u} = \gamma^2 + (1-\gamma^2) \hat{w} X^{-1} \hat{w},
\]
so $\hat{w} X^{-1} \hat{w} = 1$.
Thus
\begin{equation}
(v^T X^{-1} v) (u^T X^{-1} u) - (v^T X^{-1} u)^2 = \alpha^2 \beta^2 (1 - \gamma^2).
\label{eq:smabexact}
\end{equation}
And
\begin{align*}
(v^T M v) (u^T M u) - (v^T M u)^2
& =
 \alpha^2 \beta^2 \left[ 
\hat{v}^T M \hat{v}  (\gamma \hat{v} + \sqrt{1 - \gamma^2} \hat{w} )^T
  M (\gamma \hat{v} + \sqrt{1 - \gamma^2}) \hat{w}) \right.
\\
& \qquad \qquad \left. - \left(\hat{v}^T M  (\gamma \hat{v} + \sqrt{1 - \gamma^2} \hat{w})\right)^2
\right] \\
& =
 \alpha^2 \beta^2 (1-\gamma^2) \left[ 
(\hat{v}^T M \hat{v}) (\hat{w}^T M \hat{w}) - (\hat{v}^T M \hat{w})^2
\right] \\
& = 
\alpha^2 \beta^2 (1-\gamma^2) \left[ 
(1+ \hat{v}^T E \hat{v}) (1 + \hat{w}^T E \hat{w}) - (\hat{v}^T E \hat{w})^2
\right].
\end{align*}
Thus
\begin{align*}
 \left|
(v^T X^{-1} v) (u^T X^{-1} u) 
\right.
 & 
  - (v^T X^{-1} u)^2
 -
 \left.
 \left(
(v^T M v) (u^T M u) - (v^T M u)^2
\right)
 \right| 
\\
& =
\alpha^2 \beta^2 (1-\gamma^2)
\abs{
1 - \left( (1+ \hat{v}^T E \hat{v}) (1 + \hat{w}^T E \hat{w}) - (\hat{v}^T E \hat{w})^2 \right)
}\\
& = 
\alpha^2 \beta^2 (1-\gamma^2)
\abs{
\hat{v}^T E \hat{v} + \hat{w}^T E \hat{w}  +( \hat{w}^T E \hat{w}) ( \hat{v}^T E \hat{v} )  - (\hat{v}^T E \hat{w})^2
} \\
&
\leq
\alpha^2 \beta^2 (1-\gamma^2)
2(\tau + \tau^2).
\end{align*}
To establish the final inequality, we used that $\norm{ X^{1/2} E X^{1/2} } \leq \tau$,
and hence
\[
\abs{\hat{v}^T E \hat{w}}
\leq \tau \abs{\hat{v}^T X^{-1} \hat{w}}  \leq \tau.
\]
Combined with Equation~\eqref{eq:smabexact}, this proves the claim.
\end{proofof}

\subsection{Starting Point}
\label{sec:startpoint}
\begin{mdframed}
\vspace{\algtopspace}
\captionof{table}{\label{alg:feastart} Algorithm \feasiblestart: Given an instance $(G,y)$, outputs feasible starting point  $(x_{0},t_{0})$.\hfill \hfill}
    \begin{tight_enumerate}
      \item Use a linear time DFS to compute a topological sort on $G$ to order vertices in a sequence $(v_{1},\ldots,v_{n})$,
      	s.t. for every edge $(v_{i},v_{j})$, $i < j$.
      \item \FOR $i \assign 1,\ldots,n:$ \\
      $x_{0}(v_{i}) \assign i / n$.
      	
      \item \FOR $i \assign 1,\ldots,n:$ \\
      	$t_{0}(v_{i}) \assign \abs{x_{0}(v_{i}) - y(v_{i})}^{p} + 1 $.
      \end{tight_enumerate}
\end{mdframed}

We prove the following claim, which in turn will help us prove Lemmas~\ref{lem:boundeddom} and \ref{lem:startsym}.

\begin{claim}
\label{clm:startxy}
Let $(x_0,t_0)$ be the point returned by \feasiblestart.
For every vertex $v$,
\[
0\leq x_{0}(v) \leq 1.
\]
\end{claim}

\begin{proof}
Follows immediately from the \feasiblestart algorithm.
\end{proof}

\begin{proofof}{of Lemma~\ref{lem:boundeddom}}
First we consider another program minimizing a linear objective over a convex domain.
 \begin{equation}
\label{eq:linisoprog}
\begin{aligned}
\underset{x,t}{\min} &\quad \pair{\vecw,t} \\
\text{s.t.} & \quad (x,t) \in D_{G} \intersect (\calI_{G} \times \Reals{V})
\end{aligned}
\end{equation}
Let \optlin denote the optimal value of program~\eqref{eq:linisoprog}.
  The value \optlin is attained only when $t(v) = \abs{x(v) - y(v)}^{p}$ for every vertex $v$,
  and when this holds, the program is exactly identical to program~\eqref{eq:isoprog}.
  Hence $\optlin = \optiso$.

Now observe that the point $(x_0,t_0)$ computed \feasiblestart is feasible for program~\eqref{eq:linisoprog}. This is true because the topological sort ensures that for every edges $(a,b)$, the indices $i_{a}$ and $i_{b}$ assigned to vertices $a$ and $b$ satisfy $i_{a} < i_{b}$ and hence $x(b) - x(a) = \frac{1}{n} ( i_{a} - i_{b} ) > 0$. Meanwhile, the assignment $t_{0}(v_{i}) = \abs{x_{0}(v_{i}) - y(v_{i})}^{p} + 1 $ ensures that constraints on $t$ are not violated.
  By Claim~\ref{clm:startxy}, $\pair{\vecw,t_{0}} \leq 2 n \wmaxp < \valsup = 3 n \wmaxp$.
  Hence  $(x_0,t_0)$ is also feasible for program~\eqref{eq:bndisoprog}.
  Thus, the domain $\bnddom$ is non-empty, as $(x_{0},t_{0})$ is contained in it.
  Let $(x^{*},t^{*})$ be a feasible, optimal point for program~\eqref{eq:linisoprog},
  then clearly $\pair{\vecw,t^{*}} \leq \pair{\vecw,t_{0}} < \valsup $,
  so this point is feasible for program~\eqref{eq:bndisoprog},
  and thus $\optbnd \leq \optlin = \optiso$.
  And, as program~\eqref{eq:linisoprog} is a relaxation of program~\eqref{eq:bndisoprog},
  it follows that $\optbnd \geq \optlin = \optiso$. Thus $\optbnd = \optiso$.
  
  Finally, $\bnddom$ is bounded, because for each vertex $v$,
  $0 \leq t(v) \leq \valsup$, and $y(v) - \valsup^{1/p} \leq x(v) \leq y(v) + \valsup^{1/p}$.
\end{proofof}

\begin{proofof}{of Lemma~\ref{lem:startsym}}
Recall that 
\[
\sym(z,\bnddom) \defeq  \inf_{q \in \partial \bnddom} \inf \left\{ s > 0 : z + \frac{ (z - q)}{s} \in \bnddom \right\}.
\]
Hence for any norm $\norm{\cdot}$
\[
\sym(p,\bnddom) \geq  \frac{\inf_{q \in \partial \bnddom} \norm{q-p}}{\sup_{r \in \partial \bnddom} \norm{r-p}}.
\]
We use a norm given by $\norm{(x,t)} = \norm{x}_{\infty} + \norm{t}_{\infty}$. By giving upper and lower bounds on the distance from $(x_0, t_0)$ to the boundary of $\bnddom$ in this norm, we can lower bound the symmetry of this point.

\begin{align*}
  \max_{(t,x) \in \partial \bnddom}
    \norm{(x - x_{0},t-t_{0})} & = 
    \max_{(t,x) \in \partial \bnddom}
          \norm{x - x_{0}}_{\infty} + \norm{t-t_{0}}_{\infty} \\
    & \leq 2 \cdot \valsup^{1/p} + \valsup \leq 6 n \wmaxp.
\end{align*}
because for each vertex $v$, we have $0 \leq t(v) \leq \valsup$, and $y(v) - \valsup^{1/p} \leq x(v) \leq y(v) + \valsup^{1/p}$.

For every point $(x,t)$ on the boundary of $\bnddom$,
  we lower bound the minimum distance to $\norm{(x - x_{0},t-t_{0})}$
  by considering several conditions:
\begin{enumerate}
\item The value constraint $\pair{\vecone,t} \leq \valsup$ is active, i.e.\ $\pair{\vecone,t}  = \valsup$.
\label{enum:valcons}
\item $x(a) = x(b)$ for some edge $(a,b) \in E$. 
\label{enum:edgecons}
\item $\abs{x(a) - y(a)}^p = t(a) $ for some $v \in V$. 
\label{enum:objcons}
\end{enumerate}
At least one of the above conditions must hold for $(x,t)$ to be on the boundary of $\bnddom$.
  We will show that each condition individually is sufficient to lower bound the distance to $(x_{0}, t_{0})$.

\emph{Condition~\ref{enum:valcons}:}  $\pair{\vecone,t}  = \valsup$. Then 
\[
\norm{(x - x_{0},t-t_{0})}
\geq \norm{t-t_{0}}_{\infty}
\geq \frac{1}{n} \norm{t-t_{0}}_{1}
\geq \frac{1}{n} ( \norm{t}_{1} - \norm{t_{0}}_{1} )
\geq \frac{1}{n} \left( \valsup - 2 n \right) \geq \wmaxp.
\]

\emph{Condition~\ref{enum:edgecons}:}  $x(a) = x(b) = \gamma$ for some edge $(a,b) \in E$. 
 Then 
\begin{align*}
\norm{(x - x_{0},t-t_{0})}
&\geq 
\norm{x-x_{0}}_{\infty} \\
&\geq
\frac{1}{2}\left( \abs{x(b)-x_{0}(b)}+\abs{x(a)-x_{0}(a)} \right)\\
& =
\frac{1}{2}\left( \abs{\gamma-x_{0}(b)}+\abs{\gamma-x_{0}(a)} \right)\\
& \geq
\frac{1}{2}\left( \abs{x_{0}(b)-x_{0}(a)} \right)
\geq
\frac{1}{2n}.
\end{align*}

\emph{Condition~\ref{enum:objcons}:}  $\abs{x(a) - y(a)} = t(a)^{1/p} $ for some $a \in V$. We consider two cases.
First case is when ${\norm{t-t_{0}}_{\infty} \geq 1/2}$. This immediately implies $\norm{(x - x_{0},t-t_{0})} \geq 1/2$.

In the second case is when $\norm{t-t_{0}}_{\infty} < 1/2$.
 We write $x(a) = x_{0}(a) + \Delta$.
\begin{align*}
\abs{\Delta + x_{0}(a) - y(a)}^p 
& = t(a) 
\geq t_{0}(v) - \norm{t-t_{0}}_{\infty} \\
&\geq 1/2 + \abs{x_{0}(a) - y(a)}^{p}
\end{align*}
As $p \geq 1$, the growth rate of $\abs{\Delta + x_{0}(a) - y(a)}^{p}$
is largest when  $\abs{x_{0}(a) - y(a)}$ is maximized and as $x_{0}, y \in [0,1]$, we get  $\abs{x_{0}(a) - y(a)} = 1$, and hence  $\abs{\Delta}$ is minimized in this case. Thus $\abs{\abs{\Delta} + 1}^p \geq 1/2 + 1 = 3/2$. Consequently,
\[
 \abs{\Delta} \geq \left(\frac{3}{2}\right)^{1/p} - 1
 = \exp\left[ \frac{\log(3/2)}{p} \right] - 1
 \geq  \frac{\log(3/2)}{p} \geq \frac{1}{3p}.
\]
Thus,
\[
  \sym((x_{0}, t_{0}),\bnddom)
  \geq
  \frac{  \min(1/(3p),1/(2n))  }
  { 6 n }
  \geq
  \frac{  1 }
  { 18 n^2 p \wmaxp}.
\]

\end{proofof}


\subsection{Primal Path Following IPM with Approximate Hessian Inverse}
\label{sec:apxipm}

\begin{mdframed}
\vspace{\algtopspace}
\captionof{table}{\label{alg:isoipm}  Algorithm \isoipm: \hfill \hfill}
Run \apxipm with:\\
Objective vector $c = (0,\vecw)$ s.t. $(0,\vecw)^T (x,t) = \sum_{v \in V} \vecw(v) t(v)$. \\
Gradient function $g = g_{\bndfn}$. \\
Hessian function $M = \hsol$ with $\mu = 1/n^3.$\\
Complexity parameter $\theta(f) = \theta(\bndfn) = O(m) $. \\
Symmetry lower bound $s = \frac{1}{18 n^2 p \wmaxp}.$ \\
Value upper bound $\valsup = 3 n \wmaxp$. \\
Error parameter $\eps = \frac{\delta}{\valsup}$.\\
Starting point $(x_{0},t_{0})$ given by $\feasiblestart(G,y)$. \\
\apxipm outputs  $(x_{\text{apx}},t_{\text{apx}})$. \\
Return $x_{\text{apx}}$.
\end{mdframed}

\begin{mdframed}
\vspace{\algtopspace}
\captionof{table}{\label{alg:apxipm}
Algorithm \apxipm:
Given an objective vector $c \in \Reals{n}$,
a gradient function $g : \Reals{n} \to \Reals{n}$,
a Hessian function $M : \Reals{n} \times \Reals{n} \to \Reals{n}$,
a complexity parameter $\theta(f)$,
a feasible starting point $x_{0}$,
a symmetry lower bound $s > 0$,
a value upper bound $\valsup \geq 0$,
and an error parameter $\eps > 0$,
outputs a vector $x_{\text{apx}}$. \hfill \hfill}
    \begin{tight_enumerate}
    \item
    $x \assign x_{0}$.
    \item
    $\rho \assign 1$.
    \item
    $T_{1} \assign 20 \sqrt{ \theta(f)  } \log \left( 30 \theta(f) (1+1/s)\right)$.
   \item
    \FOR $i \assign 1,\ldots, T_{1} :$
   \item \hspace{\pgmtab}
    	$\rho \assign \rho \cdot \left( 1 - \frac{1}{20
            \sqrt{\theta(f)}} \right)$
   \item \hspace{\pgmtab}
	$z \assign - \rho g(x_{0}) + g(x)$
   \item \hspace{\pgmtab}
	$x \assign x - M(x,z) $  
    \label{alg:apxipm:p1}
    \item
    $ \alpha \assign \sqrt{ c^T M(x,c) }$
   \item 
    $\eta \assign \frac{1}{50 \alpha} $ 
   \item 
    $z \assign \eta c + g(x)$ 
   \item 
    $x \assign x - M(x,z) $
   \item 
    $T_{2} \assign 20 \sqrt{ \theta(f)  }  \log \left( \frac{ 66 \theta(f) }{\eps} \right)$.
     \label{alg:apxipm:mid}
   \item
    \FOR $i \assign 1,\ldots,T_{2}:$
   \item \hspace{\pgmtab}
    	$ \eta \assign \eta \cdot \left( 1 + \frac{1}{20
            \sqrt{\theta(f)}} \right)$
   \item \hspace{\pgmtab}
	$z \assign \eta c + g(x)$
   \item \hspace{\pgmtab}
	$x \assign x - M(x,z) $
    \label{alg:apxipm:p2}
    \item return $x_{\text{apx}} \assign x$.
    \end{tight_enumerate}
\end{mdframed}

In this section we prove Theorem~\ref{thm:pathalg}. We start by proving a central lemma shows that approximate Newton steps are sufficient to ensure convergence of our primal path following IPM.

The rest of this section is a matter of connecting this statement with Renegar's primal following machinery.

\begin{lemma}
Assume $f \in \mathcal{SC}$ and is defined on a domain $D$. If $\delta \defeq \norm{ H(x)^{-1} g(x) }_{H(x)} \leq \frac{1}{2}$, $\tau < 1$, and
\[
(1-\tau) H(x)^{-1} \preceq M \preceq (1+\tau)  H(x)^{-1}.
\]
then taking $x_{+} = x - M g(x) $ will ensure both that $x_{+} \in D$ and
\[
\norm{ H(x_{+})^{-1} g(x_{+}) }_{H(x_{+})} \leq \frac{1}{1-(1+\tau) \delta} \left( \tau \delta + \frac{((1+\tau) \delta)^2}{1-(1+\tau) \delta} \right).
\]
\label{lem:apxnewton}
\end{lemma}

\begin{proof}
For brevity write $H_{x} \defeq H(x)$.
Firstly, 
\[
\norm{ x_{+} - x}_{H_{x}} = \norm{ M g(x) }_{H_{x}} \leq (1+\tau)  \norm{ {H_{x}}^{-1} g(x) }_{H_{x}} =  (1+\tau) \delta < 1,
\]
which guarantees feasibility of $x_{+}$.
Further,
\begin{align*}
 \norm{I - M{H_{x}}}_{H_{x}}^2 & = \max_{\norm{y}_{H_{x}}=1} y^T(I-{H_{x}}M){H_{x}}(I-M{H_{x}})y \\
 &= \max_{\norm{y}_{H_{x}}=1} y^T{H_{x}}^{1/2}(I-{H_{x}}^{1/2}M{H_{x}}^{1/2})(I-{H_{x}}^{1/2}M{H_{x}}^{1/2}){H_{x}}^{1/2}y\\
 &=  \max_{\norm{y}_{H_{x}}=1}y^T{H_{x}}^{1/2}(I-{H_{x}}^{1/2}M{H_{x}}^{1/2})^2{H_{x}}^{1/2}y\\
 &\leq  \max_{\norm{y}_{H_{x}}=1} \tau^2 y^T{H_{x}}y = \tau^2
\end{align*}
Then
\begin{align*}
\norm{{H_{x}}^{-1} g(x) - M g(x) }_{H_{x}} = \norm{(I -
  M{H_{x}}){H_{x}}^{-1} g(x) }_{H_{x}} & \leq \norm{I -
                                         M{H_{x}}}_{H_{x}}
                                         \norm{{H_{x}}^{-1}g(x)
                                         }_{H_{x}} \\
& \leq \tau \norm{{H_{x}}^{-1}g(x) }_{H_{x}}.
\end{align*}
Now,
\begin{align*}
{H_{x}}^{-1} g(x_{+}) & = {H_{x}}^{-1}g(x) + \int_0^1{H_{x}}^{-1} {H}(x+t(x_{+} - x)) (x_{+} - x) \, \mathrm{d}t  \\
& = ({H_{x}}^{-1}g(x) - Mg(x)) + Mg(x) +  \int_0^1{H_{x}}^{-1} {H}(x+t(x_{+} - x)) \left(x_{+} - x\right) \, \mathrm{d}t  \\
& = ({H_{x}}^{-1}g(x) - Mg(x)) +  \int_0^1\left[ I - {H_{x}}^{-1} H(x+t(x_{+} - x))\right] M g(x) \, \mathrm{d}t  \\
\end{align*}
Thus, using Theorem~\ref{thm:schess}
\begin{align*}
\norm{{H_{x}}^{-1} g(x_{+})}_{H_{x}} & \leq \norm{ {H_{x}}^{-1}g(x) - Mg(x) }_{H_{x}}+ \norm{  \int_0^1\left[ I - {H_{x}}^{-1} H(x+t(x_{+} - x))\right] M g(x) \, \mathrm{d}t }_{H_{x}} \\
& \leq \tau \norm{{H_{x}}^{-1}g(x) }_{H_{x}} +   \int_0^1\norm{I - {H_{x}}^{-1} H(x+t(x_{+} - x))}_{{H_{x}}} \, \mathrm{d}t \norm{ M g(x) }_{H_{x}} \\
& \leq \tau \delta + (1+\tau) \delta \int_0^1\norm{I - {H_{x}}^{-1} H(x+t(x_{+} - x))}_{{H_{x}}} \, \mathrm{d}t  \\
& \leq \tau \delta + (1+\tau) \delta \int_0^1\frac{1}{(1-t (1+\tau) \delta)^2} - 1 \, \mathrm{d}t \\
& \leq \tau \delta + \frac{((1+\tau) \delta)^2}{1-(1+\tau) \delta}.
\end{align*}

Finally, we can use the self-concordance of $f$ to get
\begin{align*}
\norm{ H(x_{+})^{-1} g(x_{+}) }_{H(x_{+})} & \leq \frac{1}{1-\norm{ x_{+} - x}_{H_{x}}} \norm{ {H_{x}}^{-1} g(x_{+}) }_{{H_{x}}}\\
& \leq \frac{1}{1-(1+\tau) \delta} \left( \tau \delta + \frac{((1+\tau) \delta)^2}{1-(1+\tau) \delta} \right).
\end{align*}

\end{proof}
For completeness, we now restate several results from a textbook by Renegar~\cite{renegar_mathematical_2001}.

\begin{definition}
Consider a function $f \in \mathcal{SC}$ with bounded domain $D_{f}$. Let $\overline{D}_{f}$ be the closure of the domain. Given an objective vector $c$, we define the associated minimization problem as
\begin{equation}
\label{eq:barprog}
\begin{aligned}
& \underset{x}{\min} & & \pair{c,x} \\
& \text{subject to } x \in \overline{D}_{f},
\end{aligned}
\end{equation}
and, we define the associated $\eta$-minimization problem as
\begin{equation}
\label{eq:barprog}
\begin{aligned}
& \underset{x}{\min} & & \eta \pair{c,x} + f(x) \\
& \text{subject to } x \in D_{f}.
\end{aligned}
\end{equation}
For each $\eta$, let $z(\eta) \in D_{f}$ denote an optimum of
the $\eta$-minimization problem.
\end{definition}
Using this definition, we can state two lemmas, which are proven by Renegar,
and appear equations (2.13) and (2.14) in \cite{renegar_mathematical_2001}.
\begin{lemma}
Given a function $f \in \mathcal{SC}$ with bounded domain $D_{f}$ and an objective vector $c$, let $\opt$ denote the value of  the associated minimization problem. Then for any $\eta > 0$ and any $x \in D_{f}$
\[
\norm{H(x)^{-1} c }_{x}
\leq \pair{c,x} - \opt.
\]
\label{lem:objnormobjgap}
\end{lemma}
\begin{lemma}
Given a function $f \in \mathcal{SCB}$ with bounded domain $D_{f}$ and an objective vector $c$, let $\opt$ denote the value of  the associated minimization problem. Then for any $\eta > 0$ and any $x \in D_{f}$
\[
\pair{c,x} - \opt \leq \frac{1}{\eta} \theta(f) (1 + \norm{x - z(\eta)}_{ z(\eta)} ),
\]
where $z(\eta)$ is an optimum of the associated $\eta$-minimization problem. 
\label{lem:bigscalegoodval}
\end{lemma}
The following is a restricted form of  Renegar's Theorem 2.2.5~\cite{renegar_mathematical_2001}.
\begin{lemma}
Assume $f \in \mathcal{SC}$. If $\delta = \norm{H(x)^{-1}g(x)}_{x} \leq 1/4$ for some $x \in D_{f}$, then $f$ has a minimizer $z$ and 
\[
	\norm{z - x}_{x}
	\leq
	\delta
	+
	\frac{3 \delta^2}
	{(1- \delta)^3}.
\]
\label{lem:lowgradnearopt}
\end{lemma}
The next lemma appears in Renegar~\cite{renegar_mathematical_2001} as Proposition 2.3.7:
\begin{lemma}
Assume $f \in \mathcal{SCB}$. For all $x,y \in D_{f}$,
\[
\norm{H(y)^{-1}g(x)}_{y} \leq \left(1 + \frac{1}{\sym(x,D_{f})}\right) \theta(f).
\]
\label{lem:symgrad}
\end{lemma}

\begin{proofof}{of Theorem~\ref{thm:pathalg}}
Given a vector $v$, and $\gamma > 0$, let $ f_{v,\gamma}(x) = f(x) + \gamma \pair{v, x}$. Let
\[
n_{v,\gamma}(x) \defeq H(x)^{-1} \left( g(x) + \gamma v \right) = g_{x}(x) + \gamma H(x)^{-1}  v.
\]
Now, for any $\gamma_{1}$ and $\gamma_{2}$
\[
  n_{v,\gamma_{2}}(x) =
  \frac{\gamma_{2}}
  {\gamma_{1}}
  n_{v,\gamma_{1}}(x)
  +
  \left(
  \frac{\gamma_{2}}
  {\gamma_{1}}
  -1
  \right)
  g_{x}(x).
 \]
Thus
\[
  \norm{ n_{v,\gamma_{2}}(x) }_{x}
  \leq
  \frac{\gamma_{2}}
  {\gamma_{1}}
  \norm{ n_{v,\gamma_{1}}(x) }_{x}
  +
  \abs{
   \frac{\gamma_{2}}
  {\gamma_{1}}
  -1
  }
  \sqrt{\theta(f)}.
\]
Observe that for any $\gamma$, the Hessian $H(x)$  of $f$ is also the Hessian of $f_{\gamma}$. Consequently, we have $f_{\gamma} \in \mathcal{SC}$ because $f \in \mathcal{SCB}$.
Thus by Lemma~\ref{lem:apxnewton} applied to the function $f_{v,\gamma}$,
if $\delta \defeq \norm{ n_{v,\gamma}(x) }_{H(x)} \leq \frac{1}{2}$, $\tau < 1$, and 
\[
(1-\tau) H(x)^{-1} \preceq M \preceq (1+\tau)  H(x)^{-1},
\]
then for  $x_{+} = x - M \left( g(x) + \gamma v \right)$,
we have $x_{+} \in D_{f_{v,\gamma}} = D_{f}$
and
\begin{equation}
 \norm{ n_{v,\gamma}(x_{+}) }_{x_{+}} = \norm{ H(x_{+})^{-1} ( g(x_{+})  + \gamma_{2} v)}_{H(x_{+})} \leq \frac{1}{1-(1+\tau) \delta} \left( \tau \delta + \frac{((1+\tau) \delta)^2}{1-(1+\tau) \delta} \right).
\end{equation}
Suppose we start with 
\[
 \norm{ n_{v,\gamma_{1}}(x) }_{x} \leq 1/9,
\]
And take
\[
\gamma_{2} = \left( 1 + \frac{1}{20\sqrt{\theta(f)}} \right) \gamma_{1}.
\]
Then using $\theta(f) \geq 1$, we find
\[
 \norm{ n_{v,\gamma_{2}}(x) }_{x} \leq 1/6.
\]
For $\tau = 1/10$, letting  $x_{+} = x - M \left( g(x) + \gamma_{2} v \right)$, we get 
\[
 \norm{ n_{v,\gamma_{2}}(x_{+}) }_{x_{+}} = \norm{ H(x_{+})^{-1} ( g(x_{+})  + \gamma_{2} v)}_{H(x_{+})} \leq \frac{1}{1-11/60} \left( 1/60 + \frac{(11/60)^2}{1-11/60} \right) < 1/9.
\]
Similarly, if we take 
\[
\gamma_{2} = \left( 1 - \frac{1}{20\sqrt{\theta(f)}} \right) \gamma_{1}.
\]
then 
\[
 \norm{ n_{v,\gamma_{2}}(x) }_{x} \leq 1/6.
\]
So again, taking $x_{+} = x - M \left( g(x) + \gamma_{2} v \right)$ gives
\[
 \norm{ n_{v,\gamma_{2}}(x_{+}) }_{x_{+}} = \norm{ H(x_{+})^{-1} ( g(x_{+})  + \gamma_{2} v)}_{H(x_{+})} \leq \frac{1}{1-11/60} \left( 1/60 + \frac{(11/60)^2}{1-11/60} \right) < 1/9.
\]
With these observations in mind, we are ready to prove the correctness of the \apxipm algorithm.

We refer to the \FOR loop in step~\ref{alg:apxipm:p1} as \emph{phase 1} of the algorithm. In phase 1, we take $ v_{1} = - g(x_{0}) $, so
\[
n_{v_{1},\rho}(x) \defeq H(x)^{-1} \left( g(x) - \rho g(x_{0}) \right).
\]
Initially, as $x = x_{0}$, 
so as $\rho = 1$, we $\norm{ n_{v_{1},\rho}(x) }_{x} = 0 \leq 1/9$.
Thus, by our observations on decreasing $\gamma$, we find that after each iteration of the \FOR loop, we get  $\norm{ n_{v_{1},\rho}(x) }_{x} \leq 1/9$, and after the $i^{\text{th}}$ iteration of the \FOR loop, we get $\rho \leq \left( 1 - \frac{1}{20 \sqrt{\theta(f)}} \right)^i$. When the \FOR loop completes, we thus have
\[
\rho \leq \left( 1 - \frac{1}{20 \sqrt{\theta(f)}} \right)^{   20 \sqrt{ \theta(f)  } \log \left( 30 \theta(f) (1+1/s)\right) } \leq \frac{1}{30  \theta(f)  (1+1/s)}.
\]
Hence, for the $x$ obtained at the end of phase 1, by applying Lemma~\ref{lem:symgrad} and our symmetry lower bound $s$, we get
\begin{align*}
  \norm{ H(x)^{-1} g(x) }_{x}
  & = \norm{ \rho H(x)^{-1} g(x_{0}) + n_{v_{1},\rho}(x) }_{x} \\
  & \leq   \rho \norm{ H(x)^{-1} g(x_{0}) }_{x}  + \norm{ n_{v_{1},\rho}(x) }_{x} \\
  & \leq \rho \theta(f) (1+1/s) + 1/9 \leq 1/30 + 1/9 = 13/90.
\end{align*}

We refer to steps~\ref{alg:apxipm:mid} and~\ref{alg:apxipm:p2} as \emph{phase 2}. 
In phase 2, we consider
\[
n_{c,\eta}(x) \defeq H(x)^{-1} \left( g(x) + \eta c \right).
\]
Using $\sqrt{ c^T M c } \geq \sqrt{ \frac{9}{10} c^T H(x)^{-1} c } \geq \frac{9}{10} \norm{H(x)^{-1} c}_{x}$,
we get that at the start of step~\ref{alg:apxipm:mid}, 
\begin{align*}
  \norm{ n_{c,\eta}(x) }
  & = \norm{ \eta H(x)^{-1} c + H(x)^{-1} g(x)}_{x} \\
  & \leq   \eta \norm{ H(x)^{-1} c }_{x}  + \norm{H(x)^{-1} g(x)}_{x} \\
  &  \leq  \frac{1}{45} +13/90 = 1/6.
\end{align*}
Hence, at the end of step~\ref{alg:apxipm:mid}, we get $\norm{ n_{c,\eta}(x) }_{x} \leq 1/9$.
Thus, at the end of each iteration of the \FOR loop in step~\ref{alg:apxipm:p2},
we also get $\norm{ n_{c,\eta}(x) }_{x} \leq 1/9$.

So once the loop completes, 
using
$\sqrt{ c^T M c } \leq  \frac{11}{10} \norm{H(x)^{-1} c}_{x}$,
and that by Lemma~\ref{lem:objnormobjgap} $\norm{H(x)^{-1} c}_{x} \leq \valsup - \opt$,
we have
\[
\eta
\geq
\frac{1}{55\norm{H(x)^{-1} c}_{x}}
\left( 1 + \frac{1}{20 \sqrt{\theta(f)}} \right)^{ 20 \sqrt{ \theta(f)  }  \log \left( \frac{ 66 \theta(f) }{\eps} \right) }
\geq \frac{6 \theta(f)}{5 \eps (\valsup - \opt) }.
\]
Now from $\norm{ n_{c,\eta}(x) }_{x} \leq 1/9$ and Lemma~\ref{lem:lowgradnearopt}
applied to $f_{c,\eta}$,
we get that $\norm{x - z(\eta)}_{x} \leq 1/9+3 (1/9)^2/(1-1/9)^3 \leq 1/6$,
and by the self-concordance of $f$,
$\norm{x - z(\eta)}_{z(\eta)} \leq (1/6)/(1-1/6) = 1/5$.
Then by Lemma~\ref{lem:bigscalegoodval} applied to $f$, we have
\[
\pair{c,x} - \opt
\leq 
\frac{\theta(f)}
{\eta}
(1+\norm{x - z(\eta)}_{z(\eta)})
\leq
\eps \cdot (\valsup - \opt).
\]

\end{proofof}



\newcommand{\moddijkstra}{{\sc ModDijkstra}}
\newcommand{\steepestpath}{{\sc SteepestPath}}

\newcommand{\compvlow}{{\sc CompVLow}}
\newcommand{\compvhigh}{{\sc CompVHigh}}
\newcommand{\comphighpressgraph}{{\sc CompHighPressGraph}}
\newcommand{\fixabovepress}{{\sc FixAbovePress}}
\newcommand{\fixpathsabovepress}{{\sc FixPathsAbovePress}}
\newcommand{\vertexsteepestpath}{{\sc VertexSteepestPath}}
\newcommand{\starsteepestpath}{{\sc StarSteepestPath}}
\newcommand\fixgradzero{\ensuremath{\mathsf{AssignWithZeroGradient}}}

\newcommand{\parent}{\ensuremath{\mathsf{parent}}}
\newcommand{\LParent}{\ensuremath{\mathsf{LParent}}}
\newcommand{\HParent}{\ensuremath{\mathsf{HParent}}}
\newcommand{\temp}{\ensuremath{\mathsf{temp}}}
\newcommand{\vLow}{\ensuremath{\mathsf{vLow}}}
\newcommand{\vHigh}{\ensuremath{\mathsf{vHigh}}}
\newcommand{\treeRoot}{\ensuremath{\mathsf{root}}}
\newcommand{\pressure}{\ensuremath{\mathsf{pressure}}}
\newcommand{\sfd}{\ensuremath{\mathsf{d}}}

\newcommand\fixpath{\ensuremath{\mathsf{fix}}}
\newcommand{\dis}{{\mathsf{dist}}}
\newcommand{\interior}{\operatorname{int}}

\section{Inf and Lex minimization on DAGs}
\label{sec:appLex}
In this section, we show that given a partially labeled DAG $(G,v_0)$, we can find an inf-minimizer in $O(m)$ time and a lex-minimizer in $O(mn)$ time. 
\paragraph{Notations and Convention.}
We assume that $G = (V,E,\len)$ is a DAG and the vertex set is denoted
by $V = \{ 1,2,...,n\}.$ We further assume that the vertices are {\bf
  topologically sorted}. 
A topological sorting of the vertices can be computed by a well-known
algorithm in $O(m)$ time.
This means that if $(i,j) \in E$, then $i
<j$.
$\len : E \to \rea_{\geq 0}$ denotes non-negative edge lengths. For
all $x,y \in V$, by $\dis(x,y)$, we mean the length of the shortest
directed path from $x$ to $y$. It is set to $\infty$ when no such path
exists.

A \textbf{path} $P$ in $G$ is an ordered sequence of (distinct) vertices $P = (x_{0},x_1,\ldots,x_k),$ such that
  $(x_{i-1},x_i) \in E$ for $i \in [k].$ For notational convenience,
  we also refer to repeated pairs $(x,x)$ as paths. 
The \textbf{endpoints} of $P$ are
denoted by $\partial_{0} P = x_{0}, \partial_1 P = x_k.$ 
The set of \textbf{interior} vertices of $P$ is defined to be
$\interior(P) \defeq \{x_i : 0 < i < k\}.$ 
For $0 \le i < j \le k,$ we
use the notation $P[x_i:x_j]$ to denote the subpath
$(x_i,\ldots,x_j).$ The length of $P$ is 
$\len(P) \defeq \sum_{i=1}^k \len(x_{i-1},x_i).$

A function $v_{0} : V \to \rea \cup \{*\}$ is called a 
  \textbf{labeling}
 (of $G$). A vertex $x \in V$ is a \textbf{terminal} with
respect to $v_{0}$ iff $v_{0}(x) \neq *.$ 
The other vertices, for which $v_{0} (x) = *$, are \textbf{non-terminals}.  
We let $T(v_{0})$ denote the set of terminals with respect to $v_{0}.$ If
$T(v_{0}) = V,$ we call $v_{0}$ a \textbf{complete labeling} (of $G$). We
say that an assignment $v : V \to \rea \cup \{*\}$ \textbf{extends} $v_{0}$ if
$v(x) = v_{0}(x)$ for all $x$ such that $v_{0}(x) \neq *$.

Given a labeling $v_{0} : V \to \rea \cup \{*\}$, and 
  two terminals $x,y \in T(v_{0})$ for which $(x,y) \in E$,
  we define the \textbf{gradient} on $(x,y)$ due to $v_{0}$ to be \vspace{-6pt}
\[ 
\vspace{-6pt}
\grad_{G}^+[v_0](x,y) =  \max \left \{ \frac{v_0(x)-v_0(y)}{\len(x,y)}, 0 \right \}  .
\] 

Here and wherever applicable, we follow the convention $\frac{0}{0} =0$, $0 \cdot \infty = 0$ and $\frac{\text{finite number}}{\infty} =0.$ 
When $v_0$ is a complete labeling, we interpret $\grad^+_{G}[v_0]$ as a vector in
$\rea^m,$ with one entry for each edge. 

A graph $G$ along with a labeling $v$ of $G$ is called a
\textbf{partially-labeled graph}, denoted $(G,v).$
We say that a partially-labeled graph $(G,v_{0})$ is a 
  \textbf{well-posed instance} if for every vertex $x \in V$, either there is a path from 
  $x$ to a terminal $t \in T(v_{0})$ or there is a path from a terminal $t \in T(v_{0})$ to $x.$ We note that instances arising from 
  isotonic regression problem are well-posed instances and in fact satisfy a stronger condition. Every vertex lies on a terminal-terminal path.

A path $P$ in a partially-labeled graph $(G,v_{0})$ is
called a \textbf{terminal path} if
 both endpoints are terminals.
We define $\nabla^+ P(v_{0})$ to be its gradient:  \vspace{-6pt}
\[
\vspace{-6pt}
\nabla^+ P(v_{0}) \defeq \max \left \{ \frac{v_{0}(\partial_{0} P) - v_{0}(\partial_1
  P)}{\len(P)},0 \right \} .
\]
If $P$ contains no terminal-terminal edges (and hence, contains at least one non-terminal), 
  it is a \textbf{free terminal path}.

\paragraph{Lex-Minimization.}
An instance of the  {\sc Lex-Minimization} problem is described by a
partially-labeled graph $(G,v_{0}).$
The objective is to compute a complete labeling
$v : V_{G} \to \rea$ extending $v_{0}$ that lex-minimizes
$\grad^+_{G}[v]$.
We refer to such a labeling as a lex-minimizer.
Note that if $T(v_{0}) = V_{G},$ then trivially,
$v_{0}$ is a lex-minimizer.

\begin{definition}
  A \emph{steepest fixable path} in an instance $(G,v_{0})$ is a free
  terminal path $P$  that has the largest gradient $\nabla^+ P(v_{0})$
  amongst such paths. 
\end{definition}
Observe that if $P$ is a steepest fixable path with $\nabla^+ P(v_{0})
> 0$ then $P$ must
  be a simple path. 
\begin{definition}
  Given a steepest fixable path $P$ in an instance $(G,v_{0}),$
  we define $\fixpath_{G}[v_{0},P] : V_{G} \to \rea \cup \{*\}$ to be
  the labeling given by
\begin{align*}
\fixpath_{G}[v_{0},P](x) =
\begin{cases}
v_{0}(\partial_{0} P) - \nabla^+ P(v_{0}) \cdot \len_{G}(P[\partial_{0} P : x]) &
x \in \interior(P) \setminus T(v_{0}), \\
v_{0}(x) & \text{otherwise.} 
\end{cases}
\end{align*}
\end{definition}
We say that the vertices $x \in \interior(P)$ are fixed by the
operation $\fixpath[v_0,P].$ If we define
$v_1 = \fixpath_{G}[v_{0},P],$ where $P = (x_{0},\ldots,x_r)$ is the
steepest fixable path in $(G,v_{0}),$ then it is easy to argue that for every $i \in [r],$ we
have $\grad[v_1](x_{i-1},x_{i}) = \nabla^+ P.$

\subsection{Sketch of the Algorithms}\label{sec:algs}
We now sketch the ideas behind our algorithms
and give precise statements of our results. A full
description of all the algorithms is included in the appendix.

We define the \textbf{pressure} of a vertex to be the gradient of the steepest
  terminal path through it:
\[
\pressure[v_{0}](x) \defeq \max\{\nabla^+ P(v_{0})\ |\ P \textrm{ is a 
  terminal path in $(G,v_{0})$ and } x \in P\}.
\]
Observe that in a graph with no terminal-terminal edges, a free
terminal path is a \textbf{steepest fixable path} iff its gradient is equal to
the highest pressure amongst all vertices. Moreover, vertices that lie
on steepest fixable paths are exactly the vertices with the highest
pressure. For a given $\alpha \geq 0,$ in order to identify vertices with
pressure exceeding $\alpha,$ we compute vectors $\vHigh[\alpha](x)$
and $\vLow[\alpha](x)$ defined as follows in terms of $\dis$, the
metric on $V$ induced by $\ell$:
\[
  \vLow[\alpha](x) = \min_{t \in T(v_{0})} \{ v_{0}(t) + \alpha\cdot \dis(x,t)\} \qquad 
  \vHigh[\alpha](x) = \max_{t \in T(v_{0})} \{ v_{0}(t) - \alpha\cdot \dis(t,x)\}.
\]
Later in this section, we show how to find a steepest fixable path
in expected time $O(m)$ for DAGs using the notion of pressure, and
prove the following theorem about the \steepestpath\, algorithm
(Algorithm~\ref{alg:steepest-path}).
\begin{theorem}
\label{thm:steepestpath}
  Given a well-posed instance $(G,v_{0}),$
  \steepestpath$(G,v_{0})$ returns a steepest terminal path in
  $O(m)$ expected time.
\end{theorem}
By repeatedly finding and fixing steepest fixable paths, we can compute
a lex-minimizer. Theorem 3.3 in~\cite{KyngRSS15} gives an algorithm $\mathsf{MetaLex}$ that computes
lex-minimizers given an algorithm for finding a steepest fixable path
in $(G,v_{0}).$ Though the theorem is proven for undirected graphs, the same holds for directed graphs
as long as the steepest path has gradient $> 0.$

We state Theorem 5.2 from \cite{KyngRSS15}: \Rnote{TODO: prove properly?}
\begin{theorem}
\label{thm:basics:directed-meta}
\label{lem:fix-path}
Given a well-posed instance $(G,v_{0})$ on a directed graph $G$, let
$v_{1}$ be the partial voltage assignment extending $v_{0}$ obtained
by repeatedly fixing steepest fixable (directed) paths $P$ with
$\nabla P > 0.$ Then, any lex-minimizer of $(G,v_{0})$ must extend
$v_{1}.$ Moreover, every $v$ that extends $v_{1}$ is a lex-minimizer 
of $(G,v_{0})$
if and only if
for every edge
$e \in E_{G} \setminus (T(v_{1}) \times T(v_{1})),$ we have $\grad^{+}[v](e) = 0.$
\end{theorem}

When the gradient of
the steepest fixable path is equal to $0$,
there may be more than one lex-minimizing assignment to the remaining non-terminals.
But we can still label all the remaining vertices in $O(m)$ time by
a two stage algorithm so that all the new gradients are zero, and thus
by the above theorem we get a lex-minimizer.

\begin{lemma}
\label{lem:assignzerograd}
Given a well-posed instance $(G,v_{0}),$ with $T(v_{0}) \neq V_{G}$
whose steepest fixable path has gradient $0$, 
Algorithm \fixgradzero$(G,v_{0})$  runs in time $O(m)$ and returns a complete labeling
$v$ that extends $v_{0}$ and has $\grad^{+}[v](e) = 0$ for every $e \in E_{G} \setminus (T(v_{0}) \times T(v_{0}))$.
\end{lemma}
\begin{proof}
Consider a well-posed instance $(G,v_{0}),$ with $T(v_{0}) \neq V_{G}$
whose steepest fixable path has gradient $0$.
In the first stage, \fixgradzero labels all the vertices $x$ such that there is a path
from some terminal $t \in T$ to $x$. We label $x$ with the label of
the highest labeled terminal from which there is a path to $x$.
This
is the least possible label we can assign to $x$ in order to not
create any positive gradient edges. If this procedure creates any
positive gradient edges, then it would imply that the the steepest
path gradient was not $0$ to begin with, which we know is
false. Hence, this creates only $0$ gradient edges. The steepest
fixable path has zero gradient since after stage one, none of the
unlabeled vertices lie on a terminal-terminal path. In the second
stage, we label all the remaining vertices. An unlabeled vertex $x$ is
now labeled with the label of the least labeled terminal to which
there is a path from $x$. It is again easy to see that this does not
create any edges with positive gradient. The routine \fixgradzero  \, (Algorithm~\ref{alg:fixgradzero})
achieves this in $O(m)$ time. 
\end{proof}
On the basis of these results, we can prove the correctness and
running time bounds for the \complexmin\, algorithm (Algorithm~\ref{alg:comp-lex-min}) for computing a lex-minimizer.
\begin{theorem}
\label{thm:complexmin}
Given a well-posed instance $(G,v_{0}),$
\complexmin$(G,v_{0})$ outputs a lex-minimizer 
whose steepest fixable path has gradient $0$, $v$ of $(G,v_{0})$.
The algorithm runs in expected time $O(m n)$.
\end{theorem}

\subsubsection{Lex-minimization on Star Graphs}
We first consider the problem of computing the lex-minimizer on a star graph in which
  every vertex but the center is a terminal.
This special case is a subroutine in the general algorithm,
  and also motivates some of our techniques.

  Let $x$ be the center vertex, $T = L \sqcup R$ be the set of terminals, and all
  edges be of the form $(x,t)$ if $t \in R$ and  $(t,x)$ if  $t \in L$.  The initial labeling is given by $v : T \to \rea,$ and we abbreviate
  $\dis(x,t)$ by $\sfd(t) = \len(x,t)$ if $t \in R$ and $\dis(t,x)$ by $\sfd(t) = \len(t,x)$ if $t \in L$.
 
From Theorem~\ref{lem:fix-path} we know that we can determine the value
  of the lex minimizer at $x$ by finding a steepest fixable path. By
  definition, we need to find $t_{1} \in L,t_{2} \in R$ that maximize the
  gradient of the path from $t_{1}$ to $t_{2},$
  $\nabla^+(t_{1},t_{2}) \defeq \max \left \{ \frac{v(t_{1}) - v(t_2)}{\sfd(t_{2}) +
    \sfd(t_{2})}, 0 \right \}.$
  As observed above, this is equivalent to finding a terminal with the
  highest pressure.  We now present a simple randomized algorithm for
  this problem that runs in expected linear time.
\begin{theorem}\label{thm:star}
Given a pair of terminal sets $(L,R),$ an initial labeling $v:(L \sqcup R)  \to
\rea,$ and distances $\sfd : L \sqcup R \to \rea_{\geq 0},$
\starsteepestpath$(T,v,\sfd)$ returns $(t_1,t_2)$ with $t_1\in L, t_2 \in R$ maximizing
$\frac{v(t_{1}) - v(t_{2})}{\sfd(t_1) + \sfd(t_{2})},$ and runs in
expected time $O(|L \sqcup R|).$
\end{theorem}
\begin{proof}
The algorithm is described in
  Algorithm~\ref{alg:star-steepest-path} (named \starsteepestpath). 
  Given a terminal $t_{1} \in L$ (or $t_{2} \in R$), we can compute its pressure $\alpha$ along
  with the terminal $t_{2}$ such that either $\nabla^+(t_{1},t_{2}) = \alpha$
  in time $O (\sizeof{T})$ by scanning over the terminals in $R$ (or terminals in $L$).
  Now sample a random terminal $t_{1} \in L$, and a random terminal $t_{2} \in R$. Let $\alpha_1$ be the pressure of $t_1$ and $\alpha_2$ be the pressure of $t_2$, and set $\alpha = \max \{\alpha_1, \alpha_2 \}$.    We will show
  that in linear time one can then find the subset of terminals
  $T' = L^\prime \sqcup R^\prime$ such that $L^\prime \subset L, R^\prime \subset R$ whose pressure is greater than $\alpha$.  Assuming
  this, we complete the analysis of the algorithm.  If
  $L^\prime = \emptyset$ (or  $R^\prime = \emptyset$),  $t_{1}$ (or $t_2$) is a vertex with highest
  pressure. Hence the path from $t_{1}$ to $t_{3}$  (or $t_{4}$ to $t_{2}$) is a steepest
  fixable path, and we return $(t_{1},t_{3})$ (or $(t_{4}, t_{2})$). If neither
  $L^\prime \neq \emptyset$ nor $R^\prime \neq \emptyset$  the terminal with the highest pressure
  must be in $T^{\prime},$ and we recurse by picking a new random
  $t_{1} \in L^{\prime}$ and $t_{2} \in R^{\prime}$. As the size of $T'$ will halve in
  expectation at each iteration, the expected time of the algorithm on
  the star is $O(|T|)$.

  To determine which terminals have pressure exceeding $\alpha$, we
  observe that the condition
  $\exists t_{2} \in R: \alpha < \nabla^+(t_{1},t_{2}) =
  \frac{v(t_{1})-v(t_{2})}{\sfd(t_1) + \sfd(t_{2})},$
  is equivalent to
  $\exists t_{2} \in R: v(t_{2}) + \alpha \sfd(t_2) < v(t_{1}) - \alpha
  \sfd(t_{1}).$
  This, in turn, is equivalent to
  $\vLow[\alpha](x) < v(t_{1}) - \alpha \sfd(t_{1}).$ We can compute
  $\vLow[\alpha](x)$ in deterministic $O(|T|)$ time. Similarly, we can
  check if $\exists t_{2} \in L: \alpha < \nabla^+(t_{2},t_{1})$ by checking
  if $\vHigh[\alpha](x) > v(t_1) + \alpha\sfd(t_{1}).$ Thus, in
  linear time, we can compute the set $T'$ of terminals with pressure
  exceeding $\alpha$.  
\end{proof}

\subsubsection{Lex-minimization on General Graphs}

In this section we describe and prove the correctness of the algorithm
\steepestpath\, which finds the steepest fixable
path in $(G,v_{0})$ in $O(m)$ expected time.

\Rnote{Explain more}

\begin{theorem}
  For a well-posed instance $(G,v_{0})$ and a gradient value
  $\alpha \ge 0,$ \moddijkstra\, computes in time $O(m)$ a complete labeling $v$ and an
  array $\parent : V \to V \cup \{\mathsf{null}\}$
  such that, 
  $\forall x \in V_{G},$
  $v(x) = \min_{t \in T(v_{0})} \{ v_{0}(t) + \alpha\dis(t,x)\}.$
Moreover, the pointer array $\parent$ satisfies $\forall x \notin T(v_{0})$ such that 
$\parent(x) \neq {\mathsf{null}}$, $v(x) = v(\parent(x)) +
\alpha\cdot \len(\parent(x),x).$ 
\end{theorem}

\Rnote{and explain more}
As in the algorithm for the star graph, we need to
identify the vertices whose pressure exceeds a given $\alpha$.
For a fixed $\alpha,$ we can compute $\vLow[\alpha](x)$ and
$\vHigh[\alpha](x)$ for all $x \in V_{G}$ using topological ordering in $O(m)$ time. We describe the
algorithms \compvhigh, {\compvlow} for these tasks in
Algorithms~\ref{alg:vlow} and~\ref{alg:vhigh}. 
  \begin{corollary}
\label{cor:vlowvhigh}
  For a well-posed instance $(G,v_{0})$ and a gradient value
  $\alpha \ge 0,$ let $(\vLow[\alpha],\LParent) \leftarrow$
  \compvlow$(G,v_{0},\alpha)$ and
  $(\vHigh[\alpha],\HParent) \leftarrow$ \compvhigh$(G,v_{0},\alpha).$
  Then, $\vLow[\alpha], \vHigh[\alpha]$ are complete labeling of
   $G$ such that, $\forall x \in V_{G},$
\begin{align*}
  \vLow[\alpha](x) = \min_{t \in T(v_{0})} \{ v_{0}(t) + \alpha\cdot \dis(x,t)\} &&
  \vHigh[\alpha](x) = \max_{t \in T(v_{0})} \{ v_{0}(t) - \alpha\cdot \dis(t,x)\}.
\end{align*}
 Moreover, the pointer arrays $\LParent, \HParent$ satisfy $\forall x \notin T(v_{0}),$
$\LParent(x),\HParent(x) \neq {\mathsf{null}}$ and 
\begin{align*}
\vLow[\alpha](x) & = \vLow[\alpha](\LParent(x)) +
\alpha\cdot \dis(x,\LParent(x)), \\
\vHigh[\alpha](x) & = \vHigh[\alpha](\HParent(x)) -
\alpha\cdot \dis(\HParent(x),x).
\end{align*}
\end{corollary}
The following lemma
encapsulates the usefulness of $\vLow$ and $\vHigh.$
\begin{lemma}
\label{lem:vlowvhigh-pressure}
For every $x \in V_{G},$ $\pressure[v_{0}](x) > \alpha$
iff $\vHigh[\alpha](x) >\vLow[\alpha](x).$
\end{lemma}
 \begin{proofof}{of Lemma \ref{lem:vlowvhigh-pressure}}
\[
\vHigh[\alpha](x) >\vLow[\alpha](x)
\]
is equivalent to
 \[
 \max_{t \in T(v_{0})} \{ v_{0}(t) - \alpha\cdot \dis(t,x)\} > \min_{t \in T(v_{0})} \{ v_{0}(t) + \alpha\cdot \dis(x,t)\},
\]
which implies that there exists terminals $s,t \in T(v_0)$ such that
\[
 v_{0}(t) - \alpha\cdot \dis(t,x) > v_{0}(s) + \alpha\cdot \dis(x,s)
\]
thus, 
\[
\pressure[v_{0}](x) \geq \frac{ v_{0}(t) - v_{0}(s) }{\dis(t,x)+ \dis(x,s)} > \alpha.
\]
So the inequality on \vHigh\, and \vLow\, implies that pressure is strictly greater than $\alpha$.
On the other hand, if $\pressure[v_{0}](x) > \alpha$, there exists terminals $s,t \in T(v_0)$ such that
\[
\frac{ v_{0}(t) - v_{0}(s) }{\dis(t,x)+ \dis(x,s)} = \pressure[v_{0}](x) > \alpha.
\]
Hence, 
\[
 v_{0}(t) - \alpha\cdot \dis(t,x) > v_{0}(s) + \alpha\cdot \dis(x,s)
\]
which implies $\vHigh[\alpha](x) >\vLow[\alpha](x)$.
\end{proofof}
It immediately follows from Lemma~\ref{lem:vlowvhigh-pressure} and Corollary~\ref{cor:vlowvhigh} that the algorithm
\comphighpressgraph\, described in
Algorithm~\ref{alg:comp-high-press-graph} computes the vertex induced
subgraph on the vertex set $\{x \in V_{G}| \ \pressure[v_{0}] (x) >
\alpha\}$, which proves the corollary stated below.
\begin{corollary}
\label{cor:comphighpress}
  For a well-posed instance $(G,v_{0})$ and a gradient value
  $\alpha \ge 0,$
\comphighpressgraph$(G,v_{0},\alpha)$
 outputs a minimal induced subgraph $G^{\prime}$ of
  $G$ where every vertex $x$ has $\pressure[v_{0}](x) > \alpha.$
\Rnote{Is this right? This sounds like we're not getting rid of edges?
Oh, maybe that's ok, with an appropriate defn of edge pressure} 
\end{corollary}
We now describe an algorithm \vertexsteepestpath\, that
  finds a terminal path $P$ through any vertex $x$
  such that
  $\nabla^+ P(v_{0}) = \pressure[v_{0}](x)$ in expected $O(m)$
time. 
\begin{theorem}\label{thm:vertexsteepest}
  Given a well-posed instance $(G,v_{0}),$ and a vertex $x \in V_{G},$
  \vertexsteepestpath$(G,v_{0},x)$ returns a terminal path $P$
  through $x$ such that $\nabla^+ P(v_{0}) = \pressure[v_{0}](x)$ in
  $O(m)$ expected time.
\end{theorem}

\Rnote{TODO:  make into a proof of \steepestpath}
We can combine these algorithms into an algorithm \steepestpath\, that finds the steepest fixable
path in $(G,v_{0})$ in $O(m)$ expected time. We may assume that there are
no terminal-terminal edges in $G.$ We sample an edge $(x_{1},x_{2})$
uniformly at random from $E_{G}$, and a terminal $x_{3}$ uniformly at
random from $V_{G}.$ 
For $i=1,2,3,$ we compute the steepest terminal
  path $P_{i}$ containing $x_{i}.$
By Theorem~\ref{thm:vertexsteepest}, this
  can be done in $O(m)$ expected time. 
Let $\alpha$ be the largest gradient $\max_{i} \nabla^+ P_{i}.$ 
As mentioned above, we
can identify $G^{\prime},$ the induced subgraph on vertices $x$ with
 pressure exceeding $\alpha$, in $O(m)$ time. 
If $G^\prime$ is empty, we know that the path $P_{i}$ with largest gradient is
  a steepest fixable path. 
If not, a steepest
  fixable path in $(G,v_{0})$ must be in $G^{\prime},$ and hence we can
  recurse on $G^{\prime}.$ 
Since we picked a uniformly random edge, and
  a uniformly random vertex, the expected size of $G^{\prime}$ is at most half
  that of $G$. Thus, we obtain an expected running time
of $O(m).$
This algorithm is described in detail in Algorithm~\ref{alg:steepest-path}.

\subsubsection{Linear-time Algorithm for Inf-minimization}\label{sec:infAlg}
Given the algorithms in the previous section, it is straightforward to
construct an infinity minimizer. Let $\alpha^{\star}$ be the gradient
of the steepest terminal path. From 
Lemma 3.5 in~\cite{KyngRSS15} , we know that the norm of the inf
minimizer is $\alpha^{\star}$. Considering all trivial terminal paths
(terminal-terminal edges), and using \steepestpath, we can compute
$\alpha^{\star}$ in randomized $O(m)$ time. It is well
known (\cite{McShane,Whitney}) that $v_{1} \defeq \vLow[\alpha^{\star}]$
and $v_{2} \defeq \vHigh[\alpha^{\star}]$ are inf-minimizers. One slight issue occurs when a vertex $x$ does not lie on a terminal-terminal path. 
In such a case, one of $\vLow[\alpha^{\star}](x)$ or $\vLow[\alpha^{\star}](x)$ will not be finite. But the routine \fixgradzero  \, described earlier 
can be used to fix the values of such vertices.  It is
also known that $\frac{1}{2} (v_{1} + v_{2})$ is the inf-minimizer
that minimizes the maximum $\ell_\infty$-norm distance to all inf-minimizers. 
For
completeness, the algorithm is presented as Algorithm~\ref{alg:comp-inf-min}, and we have the following result.
\begin{theorem}
\label{thm:lineartimeinfmin}
Given a well-posed instance $(G,v_{0}),$ \compinfmin$(G,v_{0})$
returns a complete labeling $v$ of $G$ extending $v_{0}$ that
minimizes $\norm{\grad^+[v]}_{\infty},$ and runs in $O(m)$ expected time.
\end{theorem}

\subsection{Algorithms}
\label{sec:algorithms}

\begin{mdframed}
\vspace{\algtopspace}
\captionof{table}{ \label{alg:dijkstra}\moddijkstra$(G,v_{0},\alpha)$:
  Given a well-posed instance $(G,v_{0}),$ a gradient value
  $\alpha \ge 0,$ outputs a complete labeling $v$ of $G,$
  and an array $\parent : V \to V \cup \{\mathsf{null}\}.$ }
    \vspace{\algpostcaptionspace}
 \begin{tight_enumerate}
      \item \FOR $i =1 \, \TO \, n$
      \item \hspace{\pgmtab} \IF $v_0(i) \neq *$ then set ${v}(i) = +\infty$ \ELSE set ${v}(i) = v_0(i)$
      \item  \hspace{\pgmtab} $\parent(i) \leftarrow \mathsf{null}.$
      \item \FOR $i =1 \, \TO \, n$
       \item \hspace{\pgmtab} \FOR $j>i  : (i,j) \in E_{G}$ 
      \item  \hspace{\pgmtab}\hspace{\pgmtab} \IF ${v}(j) > v(i) + \alpha\cdot\len(i,j)$
      \item \hspace{\pgmtab}\hspace{\pgmtab}\hspace{\pgmtab} 
      Decrease ${\mathsf {v}}(j)$ to $v(i) + \alpha\cdot\len(i,j).$ 
      \item \hspace{\pgmtab} \hspace{\pgmtab}\hspace{\pgmtab} 
      $\parent(j) \leftarrow i.$ 
      \item \RETURN $(v,\parent)$
     \end{tight_enumerate}
\end{mdframed}

\begin{mdframed}
  \vspace{\algtopspace}
  \captionof{table}{\label{alg:vlow} Algorithm
    \compvlow$(G,v_{0},\alpha)$: Given a well-posed instance
    $(G,v_{0}),$ a gradient value $\alpha \ge 0,$ outputs
    $\vLow,$ a complete labeling for $G,$ and an
    array $\LParent$ $: V \to V \cup \{\mathsf{null}\}$.}
  \vspace{\algpostcaptionspace}
  \begin{enumerate}
      \item $(\vLow,\LParent) \leftarrow$ \moddijkstra$(G, v_0,\alpha)$
      \item \RETURN  $(\vLow,\LParent)$
      \end{enumerate}
   \end{mdframed}

\begin{mdframed}
  \vspace{\algtopspace}
  \captionof{table}{\label{alg:vhigh} Algorithm
    \compvhigh$(G,v_{0},\alpha)$: Given a well-posed instance
    $(G,v_{0}),$ a gradient value $\alpha \ge 0$, outputs
    $\vHigh,$ a complete labeling for $G$, and
    an array $\HParent : V \to V \cup \{\mathsf{null}\}$.}
    \vspace{\algpostcaptionspace}
    \begin{tight_enumerate}
     \item Let $G_1$ denote the graph $G$ with all edges reversed in direction.
      \item \FOR $x \in  V_{G}$
      \item \hspace{\pgmtab} \IF $x \in T(v_{0})$ \THEN $v_{1}(x) \leftarrow -v_{0}(x)$ \ELSE $v_{1}(x) \leftarrow v_{1}(x).$
      \item $(\temp,\HParent) \leftarrow$ \moddijkstra$(G_1, v_1,\alpha)$
      \item \FOR $x \in V_{G_1} : \vHigh (x) \leftarrow - \temp(x)$
      \item \RETURN $(\vHigh,\HParent)$
      \end{tight_enumerate}
    
  \end{mdframed}


\begin{mdframed}
\vspace{\algtopspace}
\captionof{table}{ \label{alg:comp-inf-min} Algorithm \compinfmin $(G,v_{0})$: Given a well-posed instance $(G,v_{0})$, outputs a complete labeling $\mathsf{{v}}$ for $G,$ extending $v_{0}$ that minimizes $\norm{\grad^+[\mathsf{{v}}]}_{\infty}$.}
\vspace{\algpostcaptionspace}
    \begin{tight_enumerate}
      \item $\alpha \leftarrow \max\{\grad^+[v_{0}](e) \ |\  e \in E_{G}
        \cap (T(v_{0}) \times T(v_{0})) \}.$
      \item $E_{G} \leftarrow E_{G} \setminus (T(v_{0}) \times T(v_{0}))$
      \item $P \leftarrow $\steepestpath$(G,v_{0}).$
      \item $\alpha \leftarrow \max\{\alpha,\nabla^+ P(v_{0})\}$
      \item  $(\vLow,\LParent) \leftarrow$ \compvlow$(G,v_{0},\alpha)$
      \item $(\vHigh,\HParent) \leftarrow$ \compvhigh$(G,v_{0},\alpha)$
      \item \FOR $x \in V_{G}$ 
      \item \hspace{\pgmtab} \IF $x \in T(v_{0})$ 
      \item \hspace{\pgmtab} \hspace{\pgmtab} \THEN $v(x)
        \leftarrow v_{0}(x)$ 
        \item \hspace{\pgmtab}  \ELSE \IF $\{ \vLow(x), \vHigh(x) \} \cap \{\infty, -\infty \} = \emptyset$ \THEN $v(x) \leftarrow \frac{1}{2} \cdot (\vLow(x) + \vHigh(x)).$
        \item \hspace{\pgmtab}  \ELSE $v(x) \leftarrow *$
        \item $v \leftarrow $ \fixgradzero$(G,v)$
        
      \item \RETURN $v$
      \end{tight_enumerate}
  \end{mdframed}

  \begin{mdframed}
\vspace{\algtopspace}
\captionof{table}{\label{alg:comp-high-press-graph} Algorithm
  \comphighpressgraph $(G,v_{0},\alpha)$:
  Given a well-posed instance $(G,v_{0}),$
  a gradient value $\alpha \ge 0$, outputs a minimal induced subgraph $G^{\prime}$ of
  $G$ where every vertex has $\pressure[v_{0}](\cdot) > \alpha.$ }
\vspace{\algpostcaptionspace}
    \begin{tight_enumerate}
      \item  $(\vLow,\LParent) \leftarrow$ \compvlow$(G,v_{0},\alpha)$
      \item  $(\vHigh,\HParent) \leftarrow$ \compvhigh$(G,v_{0},\alpha)$
      \item $V_{G^{\prime}} \leftarrow \{ x \in V_G\ |\ \vHigh(x) > \vLow(x)\ \}$
      \item $E_{G^{\prime}} \leftarrow \{ (x,y) \in E_G\ |\ x,y \in
        V_{G^{\prime}} \}.$
      \item $G^{\prime} \leftarrow (V^{\prime}, E^{\prime}, \len)$
      \item \RETURN $G^{\prime}$
      
      \end{tight_enumerate}
  \end{mdframed}
  \begin{mdframed}
\vspace{\algtopspace}
\captionof{table}{  \label{alg:steepest-path} Algorithm \steepestpath$(G,v_{0})$: Given  a well-posed instance $(G,v_{0}),$ with 
$T(v_{0}) \neq V_{G},$ 
outputs a steepest free terminal path $P$ in $(G,v_{0}).$}
    \vspace{\algpostcaptionspace}
    \begin{tight_enumerate}
      \item Sample uniformly random $e \in E_{G}.$ Let $e = (x_{1},x_{2}).$
      \item Sample uniformly random $x_{3} \in V_{G}.$
      \item \FOR $i=1$ \TO $3$
      \item \hspace{\pgmtab} $P
        \leftarrow $ \vertexsteepestpath$(G,v_{0},x_{i})$
        \item Let $j \in \argmax_{j \in \{1,2,3\}} \nabla^+ P_{j}(v_{0})$
      \item $G^{\prime} \leftarrow $ \comphighpressgraph$(G,v_{0},\nabla^+ P_{j}(v_{0}))$
      \item \IF $E_{G^{\prime}} = \emptyset,$ 
      \item \hspace{\pgmtab} \THEN \RETURN $P_{j}$ 
      \item  \ELSE \RETURN
        \steepestpath$(G^{\prime},v_{0}|_{V_{G^{\prime}}})$
      \end{tight_enumerate}
  \end{mdframed}

\begin{mdframed}
\vspace{\algtopspace}
\captionof{table}{ \label{alg:comp-lex-min} Algorithm
  \complexmin$(G,v_{0})$: Given a well-posed instance $(G,v_{0}),$
  outputs a lex-minimizer $v$ of $(G,v_{0})$.}
    \vspace{\algpostcaptionspace}
    \begin{tight_enumerate}
      \item \WHILE $T(v_{0}) \neq V_{G}$
      \item \hspace{\pgmtab} $E_{G} \leftarrow E_{G} \setminus
        (T(v_{0}) \times T(v_{0}))$
      \item \hspace{\pgmtab} $P \leftarrow $ \steepestpath$(G,v_{0})$
      \item  \hspace{\pgmtab} \IF    $\nabla^+ P = 0$ \THEN  $v_{0} \leftarrow $ \fixgradzero$(G,v_0)$
       \item \hspace{\pgmtab} \ELSE  $v_{0} \leftarrow $ \fixpath$[v_{0},P]$
      \item \RETURN $v_{0}$
      \end{tight_enumerate}
 \end{mdframed}

\begin{mdframed}
\vspace{\algtopspace}
\captionof{table}{ \label{alg:fixgradzero} Algorithm \fixgradzero$(G,v_{0})$: Given a well-posed instance $(G,v_{0}),$ with $T(v_{0}) \neq V_{G}$, outputs  a complete labeling $v_{0}$.}
    \vspace{\algpostcaptionspace}
    \begin{tight_enumerate}
    \item $ T \leftarrow T(v_0)$
    \item \FOR $i =1 \, \TO \, n : v_0(i) \neq *$
       \item \hspace{\pgmtab} \FOR $j>i  : (i,j) \in E_{G}$
      \item  \hspace{\pgmtab}\hspace{\pgmtab} \IF $v_{0}(j) < v_{0}(i)$ or $v_{0}(j) = *$
      \item \hspace{\pgmtab}\hspace{\pgmtab}\hspace{\pgmtab} 
       $ {v}_{0}(j)  \leftarrow v_{0}(i)$ \Snote{Double check this
        algorithm. I think you want to look at $ j < i$ where $(j,i)
        \in E,$ when going in the increasing topological order.} \Anote{I still think this is right. If I understand correctly, what you suggest should
        also compute the same thing. So in this iteration, we fix all vertices $x$ such that there is a path from a terminal $t \in T$ to it. And we want to fix it to the value of the maximum such terminal. I think this is what the algorithm does.}
      
       \item $ T \leftarrow T(v_0)$
      \item \FOR $i =n \, \TO \, 1 : v_0(i) \neq *$
       \item \hspace{\pgmtab} \FOR $j<i  : (j,i) \in E_{G}$ \AND $j \notin T$
      \item  \hspace{\pgmtab}\hspace{\pgmtab} \IF $v_{0}(j) > v_{0}(i)$ or $v_{0}(j) = *$
      \item \hspace{\pgmtab}\hspace{\pgmtab}\hspace{\pgmtab} 
      $ {v}_{0}(j) \leftarrow v_{0}(i)$
     \item \RETURN $v_{0}$
      \end{tight_enumerate}
 \end{mdframed}

  \begin{mdframed}
\vspace{\algtopspace}
\captionof{table}{\label{alg:vertex-steepest-path} Algorithm
  \vertexsteepestpath$(G,v_{0},x)$: Given a well-posed instance
  $(G,v_{0}),$ and a vertex $x \in V_{G}$, outputs a steepest terminal
  path in $(G,v_{0})$ through $x$.}
    \vspace{\algpostcaptionspace}
    \begin{tight_enumerate}
     \item Let $L := \{ i \in T(v_0) | \text{ there is a path from } i \text{ to } x \}$ and  $R := \{ i \in T(v_0) | \text{ there is a path from } x \text{ to } i \}$
     \item \IF $L = \emptyset$ or  $R = \emptyset$ \THEN \RETURN $(x,x)$
      \item Compute $\dis(t,x)$ for all
        $t \in L$ and $\dis(x,t)$ for all
        $t \in R$
        \item \IF $x \in T(v_0)$
        \item \hspace{\pgmtab} $y_1 \leftarrow \argmax_{y \in R}
          \frac{v_{0}(x)-v_{0}(y)}{\dis(x,y)};$ 
         \hspace{\pgmtab} $y_2 \leftarrow \argmax_{y \in L}
          \frac{v_{0}(y)-v_{0}(x)}{\dis(y,x)}$
            
        \item \hspace{\pgmtab} \IF $\frac{v_{0}(x)-v_{0}(y_1)}{\dis(x,y_1)}  \ge \frac{v_{0}(y_2)-v_{0}(x)}{\dis(y_2,x)}$
        \item \hspace{\pgmtab} \hspace{\pgmtab} \THEN \RETURN a
          shortest path from $x$ to $y_1$ 
        \item \hspace{\pgmtab}  \ELSE \RETURN a
          shortest path from $y_2$ to $x$ 
        \item \ELSE
          \item \hspace{\pgmtab} \FOR $t \in L \cup R,$
      \item    \hspace{\pgmtab} \hspace{\pgmtab} \IF $t \in L$  \THEN $\mathsf{d}(t) \leftarrow \dis(t,x)$ \ELSE $\mathsf{d}(t) \leftarrow \dis(x,t)$
      \item \hspace{\pgmtab} $(t_{1},t_{2}) \leftarrow$
        \starsteepestpath$(L ,R ,v_{0}|_{L \cup R},\mathsf{d})$
      \item \hspace{\pgmtab} Let $P_{1}$ be a shortest path from $t_{1}$ to $x.$ Let
        $P_{2}$ be a shortest path from $x$ to $t_{2}.$
      \item \hspace{\pgmtab} $P \leftarrow (P_{1},P_{2}).$ \RETURN $P.$
      \end{tight_enumerate}
 \end{mdframed}

\begin{mdframed}
\vspace{\algtopspace}
 \captionof{table}{\label{alg:star-steepest-path}
   \starsteepestpath$(L,R,v,\mathsf{d})$: Returns the steepest path in a
   star graph, with a single non-terminal connected to terminals in
   $T,$ with lengths given by $\mathsf{d}$, and labels given by $v.$}
    \vspace{\algpostcaptionspace}
    \begin{tight_enumerate}
      \item Sample $t_1$ uniformly and randomly from $L$ and $t_2$ uniformly and randomly from $R$
      \item Compute 
$t_{3} \in \argmax_{t \in R} \frac{v(t_1) - v(t)}{\mathsf{d}(t_{1}) +
  \mathsf{d}(t)}$ and 
  $t_{4} \in \argmax_{t \in L} \frac{v(t) - v(t_2)}{\mathsf{d}(t_{2}) +
  \mathsf{d}(t)}$
\item $\alpha \leftarrow \max \left \{ \frac{v(t_1) - v(t_{3})}{\mathsf{d}(t_{1}) +
  \mathsf{d}(t_{3})}, \frac{v(t_4) - v(t_{2})}{\mathsf{d}(t_{4}) +
  \mathsf{d}(t_{2})} \right \}$
\item Compute $v_{\mathsf{low}} \leftarrow \min_{t \in R} (v(t) + \alpha\cdot {\mathsf
   {d}}(t))$
\item $L^\prime \leftarrow \{ t \in L\ |\ v(t) > v_{\mathsf{low}} +
  \alpha\cdot \mathsf{{d}}(t)\}$
\item Compute $v_{\mathsf{high}} \leftarrow \max_{t \in L} (v(t) - \alpha\cdot {\mathsf
   {d}}(t))$
\item $R^\prime \leftarrow \{ t \in R\ |\ v(t) < v_{\mathsf{high}} -
  \alpha\cdot \mathsf{{d}}(t)\}$
\item \IF $L^\prime \cup R^\prime= \emptyset$ 
 \THEN \RETURN $(t_{1},t_{2})$
\item  \ELSE \RETURN \starsteepestpath$(L^\prime, R^\prime,
  v|_{L^\prime \cup R^\prime}, \mathsf{{d}}_{L^\prime \cup R^\prime})$
\Snote{This is not sufficient. You have to sample from both sides to
  ensure that both sides halve in expectation.}
    \end{tight_enumerate}
\end{mdframed}
\Anote{Isn't it sufficient identify the max pressure vertex in L? We do halve L in expectation, and once we identify the max-pressure vertex in L, we can identify the corresponding one in R with one scan. I will modify the explanation of this algorithm accordingly.}



\end{document}